\documentclass[accepted]{uai2026} 
                        
\usepackage[american]{babel}

\usepackage{natbib} 
\usepackage{mathtools} 
\usepackage{booktabs} 
\usepackage{tikz} 
\usepackage{amssymb}
\usepackage{kotex}
\usepackage{algorithm}
\usepackage{algpseudocode}
\usepackage{subcaption}
\usepackage{xcolor}
\usepackage[normalem]{ulem}

\usetikzlibrary{positioning}

\usepackage{amsthm}

\newtheorem{theorem}{Theorem}[section]
\newtheorem{lemma}[theorem]{Lemma}

\newtheorem{assumption}{Assumption}

\def\pa{\operatorname{Pa}}

\def\ch{\operatorname{Ch}}

\def\supp{\operatorname{supp}}
\newcommand{\skel}{\mathrm{skel}}

\makeatletter
\algnewcommand{\LineComment}[1]{\Statex \hskip\ALG@thistlm \(\triangleright\) #1}
\makeatother

\title{Relaxed Sparsest-Permutation Formulation for Causal Discovery at Scale}
\author[1]{Sunmin Oh}
\author[*,4,5]{Sang-Yun Oh}
\author[*,1,2,3]{Gunwoong Park}

\affil[1]{%
    Department of Statistics\\
    Seoul National University\\
    South Korea
}
\affil[2]{%
    Interdisciplinary Program in Artificial Intelligence\\
    Seoul National University\\
    South Korea
}
\affil[3]{%
    Institute for Data Innovation in Science\\
    Seoul National University\\
    South Korea
  }  
\affil[4]{%
    Department of Statistics and Applied Probability\\
    University of California Santa Barbara\\
    CA, USA
  }  
\affil[5]{%
    Scientific Data Division\\
    Lawrence Berkeley National Laboratory\\
    CA, USA
  }  
  
\begin{document}

\maketitle

\begin{abstract}
    Despite the growing availability of large datasets, causal structure learning remains computationally prohibitive at scale. We revisit sparsest-permutation learning for linear structural equation models and show that exact Cholesky factorization is unnecessary for structure recovery. This observation motivates a support-level relaxation that searches for sparse triangular factors over a precision-support screening graph. The relaxed formulation can be efficiently evaluated via masked zero-fill incomplete Cholesky factorization, enabling scalable comparison of candidate orderings. At the population level, we establish soundness for Markov equivalence class (MEC) recovery under no-cancellation and sparsest Markov representation assumptions, as well as robustness to ordering misspecification. Motivated by these guarantees, we introduce \textsc{SCOPE}, a sparse-Cholesky pipeline that provides a scalable implementation of the relaxed formulation. Experiments on synthetic and real datasets demonstrate that \textsc{SCOPE} matches the MEC recovery accuracy of substantially slower baselines, while achieving significantly reduced runtime and scaling to $10^4$ variables.
\end{abstract}

\vspace{-2mm}
\section{Introduction}
\label{sec:introduction}

Learning a directed acyclic graph (DAG) and its Markov equivalence class (MEC) from observational data is a central goal in causal discovery. In modern applications such as genomics, neuroscience, and socio-economic systems, this task increasingly arises at large scales involving tens of thousands of variables (e.g., \cite{linnarsson2018scRNA,trapnell2014scRNA}). At such scales, existing approaches face a fundamental tension between statistical exactness and computational scalability, as enforcing exact conditional independence or likelihood-based characterizations quickly becomes computationally prohibitive (e.g., \cite{verma1990equivalence,andersson1997mec,kalisch2007pc,chickering2002ges,colombo2012rfci}). 

\begin{figure}[t]
  \centering
  \includegraphics[width=0.96\linewidth]{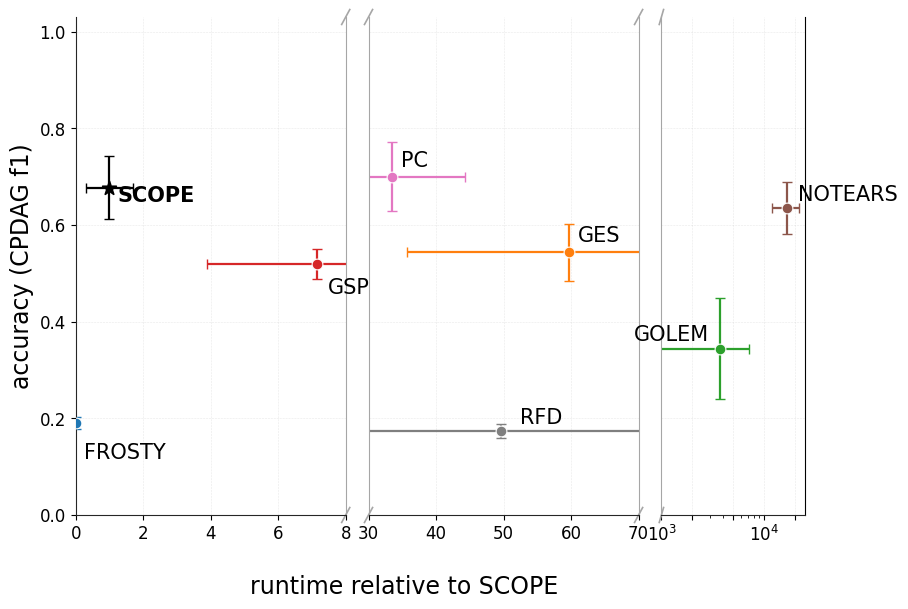}
  \vspace{-3mm}
  \caption{
  Runtime-accuracy tradeoffs for MEC recovery across baseline causal discovery methods (means $\pm$ error bars) on synthetic Gaussian linear SEMs with bounded in-degree $d=5$, $p=100$, and $n=20p$.
  }
  \label{fig:accuracy_runtime}
\end{figure} 

Figure~\ref{fig:accuracy_runtime} empirically illustrates this fundamental tension.
Representative baseline methods---PC~\citep{kalisch2007pc}, GES~\citep{chickering2002ges}, GSP~\citep{solus2021consistency}, NOTEARS~\citep{zheng2018notears}, GOLEM~\citep{ng2020golem}---enforce exact numerical formulations to achieve strong accuracy at substantially higher computational cost.
Related precision-based sparse-Cholesky pipelines, such as RFD~\citep{squires2020rfd} and FROSTY~\citep{bang2023frosty}, can be less accurate in these regimes.
Motivated by this observation, we focus on recovering the sparse structural support that is essential for causal discovery rather than insisting on exact numerical factorization. The proposed algorithm, \textsc{SCOPE}, attains strong accuracy with significantly lower runtime in Figure~\ref{fig:accuracy_runtime}.

\vspace{-2mm}
\subsection{Contributions}
\vspace{-1mm}

We revisit sparsest-permutation (SP) learning for linear structural equation model (SEM)~\citep{raskutti2018sp}. The SP approach searches over all possible orderings of variables to find the Cholesky factor $L_\pi$ of the precision matrix $\Omega_\pi$ \citep{pourahmadi1999} with the smallest number of nonzeros as shown below:
\[
\begin{aligned}
\text{Exact SP:}\quad
&\min_{\pi}\ \|L_\pi\|_0
\;\; \text{s.t.}\;\; 
\Omega_\pi = L_\pi L_\pi^\top,\\
\text{Ours:}\quad
&\min_{\pi}\ \|L_\pi\|_0
\;\; \text{s.t.}\;\; 
\Omega_\pi \circ M_\pi = (L_\pi L_\pi^\top) \circ M_\pi.
\end{aligned}
\]
As an alternative, we present a Relaxed Sparsest-Permutation (ReSP) formulation for causal structure learning that requires the coefficients to equal the precision matrix only at locations indicated by $M_\pi$.

In the remainder of this paper, we detail the theoretical and conceptual motivation for our ReSP formulation and its significant practical benefits. Our main contributions are:

\begin{itemize}

\item \textbf{ReSP: Relaxation of SP learning to precision support.}
We show that the exact Cholesky constraint $\Omega_\pi\!=\!L_\pi L_\pi^\top$ is not essential for structure recovery,
and introduce a support-level relaxation that searches for sparse triangular factors under a fixed mask.

\item \textbf{Scalable sparse-Cholesky pipeline for ReSP.}
We propose \textsc{SCOPE}, which evaluates the ReSP formulation over a candidate ordering set via masked zero-fill incomplete Cholesky factorization, followed by a lightweight refit-and-test pruning step and a sparsity-based selection step.

\item \textbf{Theory.}
We prove population-level soundness of the ReSP formulation for MEC recovery under no-cancellation and sparsest Markov representation assumptions, and establish a robustness result under ordering misspecification.
Since the analysis is formulated in terms of $\Omega$, it also extends directly to the covariance-equivalence class beyond Gaussianity.

\item \textbf{Scalability.}
Experiments on synthetic and real-world data show that \textsc{SCOPE} remains competitive while scaling to graphs with up to $p=10^4$ variables.

\end{itemize}
Detailed proofs are deferred to Appendix~\ref{app:proofs}.

\vspace{-1mm}
\section{Preliminaries}
\label{sec:preliminaries}
\vspace{-1mm}
\subsection{Linear Structural Equation Models}
\label{Sec:LSEM}

The model structure is represented by a DAG $G = (V,E)$, where
$V=\{1,\dots,p\}$ indexes variables and $E\subseteq V\times V$ is the set of directed edges.
We write $(j,k)$ for an edge from node $j$ to node $k$.
For a node $k\in V$, let $\pa(k):=\{j\in V: (j,k)\in E\}$ and $\ch(k):=\{j\in V: (k,j)\in E\}$ denote its parents and children. The moralized graph $G^m$ is the undirected graph obtained by connecting any two parents of a common child and then dropping all edge directions.

Let $\pi\!=\!(\pi_1,\dots,\pi_p)\!\in\!\mathbb S_p$ be an ordering, where $\pi_i$ denotes the variable index at position $i$.
We call $\pi$ a (not necessarily unique) topological ordering if $i\!<\!j$ whenever $(\pi_i,\pi_j)\in E$.
We view $\pi$ equivalently as a map $i\!\mapsto\!\pi(i)\!:=\!\pi_i$, and for any $S\!\subseteq[p]\times[p]$ define
$\pi(S)\!:=\!\{(\pi(i),\pi(j)):(i,j)\in S\}$.
For any $A\in\mathbb R^{p\times p}$, let $A_\pi\!:=\!P_\pi A P_\pi^\top$ denote $A$ indexed by ordering positions.
In particular, $\supp(A)\!=\!\pi(\supp(A_\pi))$.

A linear SEM associates $X\in\mathbb{R}^p$ with a DAG $G$ via
\begin{equation}
\label{eq:sem_model}
    X = B X + \epsilon,
\end{equation}
where $B\!=\![B_{kj}]$ is the weighted adjacency matrix satisfying $B_{kj}\!\neq\!0$ if and only if $(j,k)$ is an edge in $G$.
The noise vector $\epsilon\!=\!(\epsilon_1,\dots,\epsilon_p)^\top$ has independent components with $\mathbb{E}[\epsilon_j]\!=\!0$ and $\mathrm{Var}(\epsilon_j)\!=\!\sigma_j^2$ for all $j\!\in\![p]$, so that $\Sigma_\epsilon=\mathrm{diag}(\sigma_1^2,\dots,\sigma_p^2)$.
The covariance of $X$ is
\[
\Sigma=(I_p-B)^{-1}\Sigma_{\epsilon}(I_p-B)^{-\top},
\]
and we denote the population precision matrix by $\Omega=\Sigma^{-1}$.

The MEC consists of all DAGs that entail the same set of conditional independences (CI)~\citep{spirtes2000cps,pearl2009causality}.
For a DAG $G$, let $\mathcal{M}(G)$ denote its MEC.
We represent $\mathcal{M}(G)$ by its CPDAG, where an edge is directed if and only if its orientation is invariant across all DAGs in $\mathcal{M}(G)$.

Under the Gaussian linear SEM, CI relations are equivalent to vanishing partial correlations and hence are characterized by the population precision matrix $\Omega$~\citep{lauritzen1996graphical}.
Equivalently, one may view the induced equivalence through the second-order structure encoded by $\Omega$, referred to as the covariance-equivalence class (CEC)~\citep{spirtes1998path,pearl2009causality}.
Since our algorithm and theory are formulated in terms of $\Omega$, the same arguments naturally extend beyond Gaussianity to the corresponding $\Omega$-based equivalence.

Throughout the paper, we consider i.i.d.\ samples from the linear SEM~\eqref{eq:sem_model} with independent Gaussian errors and $\Sigma\succ 0$, so that the population precision $\Omega\!=\!\Sigma^{-1}$ is well-defined.
We focus on regimes with small in-degree so that parent sets remain sparse, while allowing the moralized degree to be large due to hub nodes and moralization.
This regime is consistent with networks observed in genomics, neuroscience, and socio-economic systems, which often exhibit sparse local dependencies with prominent hubs and substantial degree heterogeneity~\cite{barabasi_network_2004,jeong2000metabolic,bullmore2009complex,newman2003structure}.

\vspace{-2mm}
\subsection{Sparsest-Permutation Approaches}
\label{subsec:existing}
\vspace{-1mm}

We briefly formalize sparsest-permutation (SP) learning for Gaussian linear SEMs and summarize representative algorithmic instantiations. In linear SEMs, SP learning formulates structure discovery as a variable ordering problem. Across existing approaches, a common objective is
\begin{equation}
\label{eq:sp_objective}
\min_{\pi \in \mathbb{S}_p} \ \|L_\pi\|_0
\quad \text{s.t.} \quad
\Omega_\pi = L_\pi L_\pi^\top,
\end{equation}
where $L_\pi$ is the lower-triangular Cholesky factor of $\Omega_\pi\!:=\!P_\pi \Omega P_\pi^\top$.
The off-diagonal support of $L_\pi$ induces a directed edge set under $\pi$.
Specifically, define
\[
\widetilde E(\pi,L_\pi)
:=
\{\,(\pi(j),\pi(k)):\ k>j,\ (L_\pi)_{kj}\neq 0\,\},
\]
and let $G_\pi=(V,\widetilde E(\pi,L_\pi))$ be the induced DAG.
By construction, the off-diagonal nonzeros of $L_\pi$ are in one-to-one correspondence with $\widetilde E(\pi,L_\pi)$, so
\vspace{-2mm}
\[
\|L_\pi\|_0 \;=\; p + |\widetilde E(\pi,L_\pi)|.
\]

\vspace{-2mm}
Identifiability for MEC recovery can be obtained under sparsest Markov representation (SMR) assumption, which is typically weaker than faithfulness; see~\citet{raskutti2018sp}.
The distinction among SP-based methods lies in how the sparsity score $\|L_\pi\|_0$ is evaluated or approximated.

\textbf{Exact SP (CI-based).}
The original SP objective treats zero identification as a CI problem and enforces sparsity entrywise~\citep{raskutti2018sp}:
\vspace{-1mm}
\[
\text{Exact SP:}\quad
\begin{cases}
\displaystyle \min_{\pi \in \mathbb{S}_p} \ \|L_\pi\|_0,\\[4pt]
L_{\pi,ij} = 0
\;\Longleftrightarrow\;
X_i \perp X_j \mid X_{S_{ij}(\pi)} .
\end{cases}
\]

\vspace{-3mm}
Equivalently, these CI statements can be assessed via partial correlations by testing whether the corresponding entry of $\Omega$ vanishes.
However, this formulation conditions on ordering-dependent sets $S_{ij}(\pi)$, that is, subsets of variables preceding $j$ under $\pi$.
Because these sets can grow linearly with $p$, the number of required CI tests is exponential in the worst case.

\textbf{Greedy SP variants.}
Greedy variants do not directly solve the global optimization problem over the full permutation space $\mathbb{S}_p$. Instead, they perform a local search over permutations using simple neighborhood moves. Starting from an initial ordering, these methods iteratively update it within a local neighborhood:
\vspace{-1mm}
\[
\text{Greedy SP:}\hspace{-1mm}
\begin{cases}
\pi^{(0)} \ \text{initialized},\\[4pt]
\pi^{(t+1)} \in \mathcal{N}(\pi^{(t)}) 
\ \text{s.t.}\ 
\mathcal{S}(G_{\pi^{(t+1)}}) < \mathcal{S}(G_{\pi^{(t)}}),\\[6pt]
\text{stop at } \pi^\star \text{ if }
\mathcal{S}(G_{\pi'}) > \mathcal{S}(G_{\pi^\star})
\ \forall \pi' \in \mathcal{N}(\pi^\star).
\end{cases}
\]

\vspace{-3mm}
Here, $\mathcal{N}(\pi)\subset\mathbb{S}_p$ denotes a local neighborhood of orderings
(e.g., adjacent swaps~\cite{teyssier2005ordering} or tucks~\cite{lam2022grasp}), and the score $\mathcal{S}$ may be chosen as an explicit sparsity proxy (e.g., $|\widetilde E(\pi,L_\pi)|$) or as a likelihood-based score with a complexity penalty
(e.g., edge-count penalties or BIC/BDeu-type criteria), so that lower-score moves tend to favor sparser induced graphs.
However, greedy SP methods still compare orderings through repeated score evaluations, rather than by explicitly
identifying individual zero entries in $L_\pi$.
Representative examples include GSP~\cite{solus2021consistency}, GRaSP~\cite{lam2022grasp}, and ARCS~\cite{ye2020optimizing}. 

\textbf{Limitations.}
Both exhaustive and greedy variants use the sparsity pattern of $L_\pi$ to compare candidate orderings, rather than as a key computational component for certifying zeros in $L_\pi$.
Evaluating each candidate still entails many ordering-dependent statistical queries, either CI tests (Exact SP) or repeated likelihood-based scoring with complexity penalties (greedy SP).
Consequently, even when the ordering search is restricted to local neighborhoods of $\mathbb{S}_p$, the overall runtime need not decrease proportionally in large-scale settings.

\vspace{-2mm}
\subsection{Sparse Cholesky vs.\ SP objectives}
\label{Sec:Limitation}

Sparse Cholesky factorization computes a lower-triangular factor $L$ satisfying $A=LL^\top$ for a sparse symmetric positive definite matrix $A$~\citep{davis2006direct,saad2003iterative}.
Its primary goal is to choose an ordering $\pi$ so that factoring the permuted matrix $A_\pi:=P_\pi A P_\pi^\top$ introduces as little \emph{fill-in} as possible, i.e., nonzeros created during elimination beyond the original sparsity pattern of $A_\pi$.

Fill-reducing approaches are typically formulated via an undirected elimination-graph objective, which captures the sparsity pattern induced during Cholesky elimination. With $A_\pi=L_\pi L_\pi^\top$, we define the fill set under $\pi$ as
\[
\mathcal{F}(\pi)
=
\supp(L_\pi+L_\pi^\top)\setminus\supp(A_\pi).
\]
To connect with SP learning, we specialize to $A\!=\!\Omega$, the population precision matrix. For Gaussian linear SEMs, SP learning can be viewed as an ordering search over $\pi\in\mathbb{S}_p$ for a sparse Cholesky factor of $\Omega_\pi$~\cite{raskutti2018sp}.

\textbf{Objective comparison.}
Although both sparse Cholesky ordering and SP learning are phrased through a triangular factor, they optimize different criteria:
\[
\text{Sparse-Chol:}\ \min_{\pi \in \mathbb{S}_p} |\mathcal{F}(\pi)|
\quad\text{vs.}\quad
\text{SP:}\ \min_{\pi \in \mathbb{S}_p} \|L_\pi\|_0.
\]
This mismatch arises because fill reduction is defined on an \emph{undirected} elimination graph, whereas $\|L_\pi\|_0$ reflects \emph{directed} sparsity from conditional independencies. Figure~\ref{fig:support_nesting} illustrates the difference in skeletons obtained under fill-reducing Cholesky orderings and under the SP criterion differ relative to the undirected edge set $\supp(\Omega)$.

\begin{figure}[t]
\centering
\resizebox{0.95\linewidth}{!}{%
\begin{tikzpicture}[font=\Large]

\definecolor{outerGray}{RGB}{242,243,245}
\definecolor{midBlue}{RGB}{224,235,245}
\definecolor{innerSand}{RGB}{205,216,182}
\definecolor{ink}{RGB}{35,35,35}

\node[draw=ink, thick, rounded corners=6pt,
      fill=outerGray,
      minimum width=12cm,
      minimum height=4.5cm] (outer) {};

\node[anchor=north west] at ([xshift=8pt,yshift=-6pt]outer.north west)
{$\pi(\supp(L_\pi + L_\pi^\top))$};

\node[anchor=north east] at ([xshift=-8pt,yshift=-6pt]outer.north east)
{\Large S1};

\node[draw=ink, thick, rounded corners=6pt,
      fill=midBlue,
      minimum width=10.5cm,
      minimum height=3.0cm,
      anchor=north] (mid)
      at ([yshift=-1.2cm]outer.north) {};

\node[anchor=north west] at ([xshift=10pt,yshift=-7pt]mid.north west)
{$\supp(\Omega)$};

\node[anchor=south east] at ([xshift=-8pt,yshift=63pt]mid.south east)
{\Large S2};

\node[draw=ink, thick, rounded corners=6pt,
      fill=innerSand,
      minimum width=6.2cm,
      minimum height=1.6cm] (inner)
      at ([yshift=-0.3cm]mid.center) {};

\node at (inner)
{$\pi_0(\supp(L_{\pi_0} + L_{\pi_0}^\top))$};

\end{tikzpicture}}
\vspace{-2mm}
\caption{Nested patterns induced by an ordering: S1 fill beyond $\supp(\Omega)$, and S2 more attainable sparsity under $\pi_0$.}
\label{fig:support_nesting}
\end{figure}
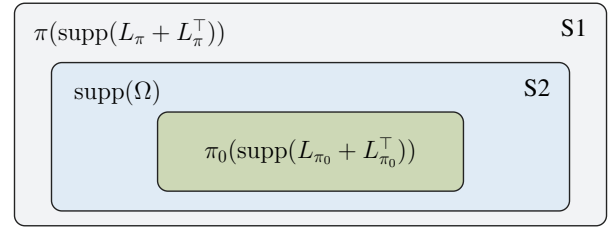

Sparse Cholesky factorization methods are designed to reduce the set difference S1 in Figure~\ref{fig:support_nesting} when computing an ordering $\pi$ and the corresponding factor $L_\pi$. For a given ordering $\pi$, the symmetrized triangular support $\pi(\supp(L_\pi\!+\!L_\pi^\top))$ may extend beyond $\supp(\Omega)$ which corresponds to \emph{fill-ins} in sparse linear algebra communities.

SP learning imposes a stricter sparsity objective. Even after excluding the set difference~S1, different orderings can induce different triangular sparsity patterns within $\supp(\Omega)$. The set difference~S2 highlights additional nonzeros that vanish under a reverse topological ordering $\pi_0$, associated with moralized edges arising from v-structures under suitable regularity conditions. 
Accordingly, SP learning accounts for sparsity both outside and within $\supp(\Omega)$.

An ordering with $|\mathcal{F}(\pi)|\!=\!0$ can still yield a denser $L_\pi$ than another ordering that incurs a small number of fill edges in the undirected elimination graph; see Appendix~\ref{app:counterexample}. Accordingly, fill reduction alone is insufficient to control the sparsity of $L_\pi$. In other words, an incorrect variable ordering results in a non-decreasing number of undesirable fill edges S1, while identifying the vanishing moralization edges S2 requires the correct ordering—an intrinsically combinatorial search over orderings. Moreover, even given a correct ordering, finite-sample estimation noise can obscure these vanishing edges and make them hard to detect reliably, rendering standard SP learning both computationally and statistically formidable in practice.

\vspace{-2mm}
\section{Relaxed Sparsest-Permutation Formulation}
\label{Sec:IC}
\vspace{-1mm}

This section introduces a Relaxed SP (ReSP) formulation for causal structure learning and establishes soundness and robustness for MEC recovery at the population level:
\vspace{-1mm}
\begin{equation}
\begin{aligned}
\label{eq:new_objective}
\min_{\pi \in \mathbb{T}_{p}} \ \|L_\pi\|_0
\;\; \text{s.t.}\;
\begin{cases}
(\Omega_\pi - L_\pi L_\pi^\top)\circ M_\pi=\mathbf{0},\\
L_\pi \circ (\mathbf{1}-M_\pi)=\mathbf{0}.\\    
\end{cases}
\end{aligned}
\end{equation}

\vspace{-3mm}
Here, $M$ denotes the precision-support mask, defined by
\[
M := \mathbf{1}\{\Omega\neq 0\}\in\{0,1\}^{p\times p}, \quad
M_{ij}= \mathbf{1}\{\Omega_{ij}\neq 0\}.
\]
For each ordering $\pi$, $\Omega_\pi:=P_\pi\Omega P_\pi^\top$, and $M_\pi:=P_\pi M P_\pi^\top$.
The set $\mathbb{T}_{p}$ collects candidate orderings, for example from prior knowledge or a user-specified ordering heuristic.
The mask $M$ yields an ordering-agnostic screening graph, restricting feasibility to $\supp(\Omega)$ across all permutations. It suppresses the set difference S1 in Figure~\ref{fig:support_nesting} by forbidding any fill outside $\supp(\Omega)$, while the formulation compares orderings by the sparsity of $L_\pi$ within the precision support.

We evaluate~\eqref{eq:new_objective} via masked zero-fill incomplete Cholesky factorization, often denoted as IC(0) in the literature~\cite{lin1999incomplete,meijerink1977ic,saad2003iterative}.
For each $\pi$ we compute the IC(0) factor $L_\pi=\mathrm{IC0}(\Omega_\pi;M_\pi)$, which satisfies
\begin{equation}
\label{eq:masked_ic0_constraints}
(\Omega_\pi-L_\pi L_\pi^\top)\circ M_\pi=\mathbf 0,
\quad
L_\pi\circ(\mathbf 1-M_\pi)=\mathbf 0.
\end{equation}
Our focus is the support of $L_\pi$, not numerical approximation of $\Omega_\pi$ by $L_\pi L_\pi^\top$.

For linear SEMs, the precision entries combine direct and shared-child effects.
To make this explicit, index variables so that $B$ is strictly lower triangular.
Then, for any $j<k$, the population precision matrix admits the decomposition
\[
\Omega_{kj}
=
-\sigma_k^{-2}B_{kj}
+\sum_{\ell>k}\sigma_\ell^{-2}B_{\ell j}B_{\ell k}.
\]
The first term corresponds to the directed edge $(j,k)$, while the summation aggregates contributions from shared children of $j$ and $k$. Consequently, absent exact algebraic cancellations, any adjacency in the moralized graph $G_0^m$ yields $\Omega_{ij}\!\neq\!0$. This motivates the following no-cancellation condition, under which $\supp(\Omega)$ provides a valid ordering-agnostic screening mask.

\begin{assumption}[No-cancellation {\cite{loh2014icov}}]
\label{assm:nocancel}
Assume the linear SEM~\eqref{eq:sem_model} with independent noise, and without loss of generality index variables so that $B$ is strictly lower triangular. For all $j<k$,
\[
\Omega_{kj}=0
\ \Rightarrow\
B_{kj}=0
\ \text{ and }\
B_{\ell j}B_{\ell k}=0,\ \forall\,\ell>k.
\]
\end{assumption}

\begin{lemma}[Mask validity under no-cancellation]
\label{lem:mask_validity}
Assume Assumption~\ref{assm:nocancel}. Then for the true DAG $G_0$ and any $\pi\in\mathbb{S}_p$,
\[
\pi\bigl(\skel(G_0)\bigr)\subseteq \supp(\Omega_\pi)=\pi(\skel(G_0^m)).
\]
\end{lemma}
\vspace{-2mm}

Lemma~\ref{lem:mask_validity} shows that, uniformly over all orderings, precision-support mask restricts admissible nonzeros without removing true parent-child adjacencies, and rules out the set difference S1 (fill beyond $\supp(\Omega)$). Under misspecified orderings, any induced candidates are confined to moralized edges. Hence, precision-support masking offers a scalable mechanism for suppressing false discoveries and is well suited to large graphs. 

ReSP may admit multiple minimizers, reflecting the fact that an equivalence class can contain multiple DAGs. We therefore impose an SMR strictness condition relative to $\mathcal{M}(G_0)$.
For the analysis, we restrict our attention to candidate set $\mathbb{T}_p$ on which the masked IC(0) construction is well-posed. Specifically, for every $\pi\in\mathbb{T}_p$, IC(0) does not break down and introduces no fill outside the precision-support mask:
\begin{equation}
\label{eq:admissible}
    (L_\pi L_\pi^\top)\circ(\mathbf 1-M_\pi)=\mathbf 0,
    \quad \forall\,\pi\in\mathbb{T}_p.
\end{equation}
We call such $\mathbb{T}_p$ IC(0)-admissible.

\begin{assumption}[SMR {\cite{raskutti2018sp}}]
\label{assm:SMR2}
Let $\pi_0$ be a reverse topological ordering of $G_0$. For each $\pi\in\mathbb{S}_p$, define $L_\pi$ via $\Omega_\pi=L_\pi L_\pi^\top$ and let $G_\pi$ be the induced DAG with edge set $\widetilde E(\pi,L_\pi)$. Assume that
\[
G_\pi\notin \mathcal{M}(G_0)
\quad\Rightarrow\quad
\|L_\pi\|_0>\|L_{\pi_0}\|_0,
\quad \forall\,\pi\in\mathbb{S}_p.
\]
\end{assumption}

\begin{theorem}[Guarantees of the ReSP formulation]
\label{thm:guarantees}
Consider a Gaussian linear SEM~\eqref{eq:sem_model} with Assumption~\ref{assm:nocancel}, and fix the precision-support mask.
\vspace{-2mm}
\begin{itemize}
\item \textbf{(Soundness).}
Suppose $\mathbb{T}_p$ is IC(0)-admissible and contains a reverse topological ordering $\pi_0$ of $G_0$, and that Assumption~\ref{assm:SMR2} holds.
Then every minimizer $(\hat\pi,L_{\hat\pi})$ of~\eqref{eq:new_objective} satisfies
\vspace{-5mm}
\[
G_{\hat\pi}\in\mathcal{M}(G_0).
\]
\vspace{-6mm}
\item \textbf{(Robustness).}
For any choice of $\mathbb{T}_p$ and any $\pi\in\mathbb{T}_p$,
every IC(0) output $L_\pi$ for~\eqref{eq:masked_ic0_constraints} satisfies
\vspace{-1mm}
\[
\widetilde E(\pi,L_\pi)\subseteq \skel(G_0^m).
\]

\vspace{-1mm}
Hence, under ordering misspecification, feasibility under the mask prevents any spillover beyond $\supp(\Omega)$.
\end{itemize}
\end{theorem}
\vspace{-1mm}

Theorem~\ref{thm:guarantees} provides best and worst case theoretical guarantees of the ReSP formulation.
When an IC(0)-admissible candidate set contains a reverse ordering and the SMR condition holds, minimizing the objective recovers a DAG in the true MEC. When such an ordering is absent in $\mathbb{T}_p$, feasibility under the precision-support mask confines all induced edges to the moralized skeleton, ruling out false discoveries outside $\supp(\Omega)$. Consequently, small candidate sets $\mathbb{T}_p$ still provide a nontrivial accuracy floor, while preserving scalability through a reduced ordering search space.

\vspace{-2mm}
\section{SCOPE: A Heuristic Algorithm}
\vspace{-1mm}

We propose \textsc{SCOPE} (Sparse-ChOlesky PipelinE), a finite-sample procedure for optimizing the ReSP formulation~\eqref{eq:new_objective} over a candidate ordering set $\mathbb{T}_p$.
\textsc{SCOPE} has four stages:
0) precision estimation,
1) relaxed-feasible factor construction,
2) refit-and-test refinement,
and 3) sparsity-based selection.

\textsc{SCOPE} solves the finite-sample counterpart of
\eqref{eq:new_objective}:
\begin{equation}
\label{eq:scope_finite_objective}
\min_{\pi\in\mathbb{T}_p} \|\widetilde L_\pi\|_0
\;\text{s.t.}
\begin{cases}
(\widehat\Omega_\pi-\widehat L_\pi \widehat L_\pi^\top)\circ \widehat M_\pi=\mathbf 0,\\[2pt]
\widehat L_\pi\circ(\mathbf 1-\widehat M_\pi)=\mathbf 0,\\[2pt]
\widetilde L_\pi=\textsc{RefitTest}(X;\supp(\widehat L_\pi)).
\end{cases}
\end{equation}
Here, $\widehat{\Omega}_\pi := P_\pi \widehat{\Omega} P_\pi^\top$ and
$\widehat{M}_\pi := P_\pi \widehat{M} P_\pi^\top$, where $\widehat{M}$ is an estimated precision-support mask constructed from $\widehat{\Omega}$.
Algorithm~\ref{alg:scope} summarizes the procedure.

\textbf{Stage 0).}
The ReSP formulation in~\eqref{eq:scope_finite_objective} is expressed in terms of a precision matrix. Accordingly, we first compute a sparse estimator $\widehat{\Omega}$ from the data, whose support determines the admissible mask $\widehat{M}=\mathbf{1}\{\widehat{\Omega}\neq 0\}$.

In principle, any consistent estimator of $\Omega$ can be used within
\textsc{SCOPE}. However, the ReSP formulation depends highly on the zero pattern of the induced triangular factor, and accurate \emph{support recovery} is prioritized over precise numerical estimation at this stage. Hence, we recommend an estimation strategy that enables thresholded support recovery under high-dimensional scaling. In our implementation, we use the graphical lasso (GL) with a distributionally robust optimization (DRO)-calibrated regularization parameter at level $\alpha^{(0)}$~\citep{friedman2008glasso,cisnerosvelarde2020droglasso}. Appendix~\ref{app:sample} provides theoretical rationale for the resulting precision estimate $\widehat{\Omega}$.

\begin{algorithm}[t]
\caption{\textsc{SCOPE}: Sparse-Cholesky Pipeline}
\label{alg:scope}
\begin{algorithmic}[1]
\Require Data $X\in\mathbb{R}^{n\times p}$, levels $\alpha^{(0)}$, $\alpha^{(2)}$, orderings $\mathbb{T}_p$
\Ensure Estimated edge set $\widehat{E}$

\State $\widehat{\Omega} \gets \textsc{EstimatePrecision}(X;\alpha^{(0)})$
\State $\widehat{M} \gets \textsc{BuildMask}(\widehat{\Omega})$
\For{$\pi \in \mathbb{T}_p$}
    \State $\widehat{\Omega}_\pi \gets P_\pi \widehat{\Omega} P_\pi^\top$
    \State $\widehat{M}_\pi \gets P_\pi \widehat{M} P_\pi^\top$ 
    \LineComment{\emph{ReSP factor evaluation in~\eqref{eq:new_objective}:}}
    \State $\widehat{L}_\pi \gets \mathrm{IC0}(\widehat{\Omega}_\pi;\widehat{M}_\pi)$
    \LineComment{\emph{Refit-and-test pruning step:}}
    \State $\widetilde{L}_\pi \gets
    \textsc{RefitTest}(X;\supp(\widehat{L}_\pi),\alpha^{(2)})$
\EndFor
\State $\widehat{\pi}\in\arg\min_{\pi\in\mathbb{T}_p}\|\widetilde{L}_\pi\|_0$
\State \Return $\widehat{E}\gets\widetilde{E}(\widehat{\pi},\widetilde{L}_{\widehat{\pi}})$
\end{algorithmic}
\end{algorithm}

\textbf{Stage 1).}
For a fixed ordering $\pi$, Stage~1 computes a masked IC(0) factor
$\widehat{L}_\pi=\mathrm{IC0}(\widehat{\Omega}_\pi;\widehat{M}_\pi)$, a lower-triangular matrix satisfying the masked feasibility
constraints in~\eqref{eq:masked_ic0_constraints} with $(\Omega_\pi,M_\pi)$ replaced by $(\widehat{\Omega}_\pi,\widehat{M}_\pi)$.

\textbf{Stage 2).}
Stage~2 is a refit-and-test pruning step that mitigates estimation bias in $\widehat{\Omega}$ from Stage~0 and numerical errors accumulated during IC(0) from Stage~1. Stage~2 is designed to prune tiny but nonzero coefficients so as to remove moralization edges in S2 when the unshielded v-structures are exposed by the ordering, decreasing $\|\widehat L_\pi\|_0$.
Let
\vspace{-1mm}
\[
\widehat{\pa}_\pi(j):=\{i<j:(j,i)\in\supp(\widehat{L}_\pi)\}
\]

\vspace{-2mm}
denote the candidate parent set of node $j$ under ordering $\pi$.
For each $j$, we refit the local regression restricted to $\widehat{\pa}_\pi(j)$,
\vspace{-1mm}
\[
\widehat{\beta}_{j,\widehat{\pa}_\pi(j)}
\in
\arg\min_{\beta\in\mathbb{R}^{|\widehat{\pa}_\pi(j)|}}
\bigl\|X_j - X_{\widehat{\pa}_\pi(j)}\beta\bigr\|_2^2,
\]

\vspace{-2mm}
and let $p_{j,i}$ be the usual (asymptotic) $t$-test $p$-value for testing $H_0:\beta_{j,i}=0$
in this refitted regression.
We then define the pruned factor $\widetilde{L}_\pi$ by entrywise masking of $\widehat{L}_\pi$:
\vspace{-1mm}
\begin{align*}
(\widetilde{L}_\pi)_{j i}
&:=
(\widehat{L}_\pi)_{j i}\,\mathbf{1}\{ i\in\widehat{\pa}_\pi(j),\ p_{j,i}\le \alpha^{(2)}\},
\quad (i<j),\\
(\widetilde{L}_\pi)_{j j}
&:=
(\widehat{L}_\pi)_{j j}.
\end{align*}

\vspace{-3mm}
This refit-and-test step follows post-selection OLS refitting ideas~\citep{belloni2012sparse,belloni2014inference},
and the subsequent test-based pruning is in the spirit of screen-and-clean procedures~\citep{wasserman2009high}.

\textbf{Stage 3).}
Finally, \textsc{SCOPE} selects an ordering
\vspace{-1mm}
\[
\widehat{\pi}\in\arg\min_{\pi\in\mathbb{T}_p}\|\widetilde{L}_\pi\|_0
\]

\vspace{-3mm}
and returns the corresponding directed edge set
\vspace{-1mm}
\[
\widehat{E}:=\widetilde{E}(\widehat{\pi},\widetilde{L}_{\widehat{\pi}})
=
\{(\widehat{\pi}(j),\widehat{\pi}(k)):\ k>j,\ (\widetilde{L}_{\widehat{\pi}})_{k j}\neq 0\}.
\]

\vspace{-3mm}
The criterion $\|\widetilde L_\pi\|_0$ may have multiple minimizers over $\mathbb{T}_p$,
as ordering-based sparsity objectives can be indifferent among orderings within the same MEC~\citep{raskutti2018sp,solus2021consistency}.
In case of ties, any minimizer may be selected; if desired, ties can be broken by a secondary score such as BIC computed from the refitted local regressions~\citep{chickering2002ges,teyssier2005ordering}.

\textbf{Remark: Well-posedness.}
Stages~1) and~3) of \textsc{SCOPE} are deterministic given their inputs. Stages~0) and~2) employ standard statistical procedures that are consistent under appropriate assumptions. Accordingly, the algorithm is well-defined in finite sample setting, with statistical guarantees inherited from the underlying components.




\vspace{-2mm}
\subsection{One-pass SCOPE with Approximate Minimum Degree Ordering}
\vspace{-1mm}

We adopt a data-driven choice $\mathbb{T}_p\!=\!\{\hat\pi_{\mathrm{AMD}}\}$ with $|\mathbb{T}_p|\!=\!1$, where $\hat\pi_{\mathrm{AMD}}$ is an approximate minimum-degree (AMD) ordering computed from $\widehat{\Omega}$~\citep{amestoy1996amd}.
This minimalist choice is justified by Theorem~\ref{thm:guarantees}, which ensures that \textsc{SCOPE} remains informative even with a singleton candidate ordering set by restricting evaluation to the precision-screened space induced by the mask.

Two additional considerations support using AMD. First, minimum-degree elimination is a classical fill-reducing heuristic for sparse triangular factors~\citep{george1989mdo,berman1990mdo}. Second, under the no-cancellation condition, the precision support coincides with the moralized skeleton (Lemma~\ref{lem:mask_validity}), so degree inflation caused by moralization tends to make minimum-degree orderings closer to reverse orderings in practice~\citep{loh2014icov}. Empirical results in Sections~\ref{sec:experiments} and~\ref{sec:real} confirm that this minimalist choice remains competitive at large scale.

\vspace{-2mm}
\subsection{Computational Complexity}
\label{subsec:complexity}
\vspace{-1mm}

We analyze the runtime of one pass of Algorithm~\ref{alg:scope} for a fixed ordering $\pi$.
Stage~0 (precision estimation and mask construction) is a plug-in component and may be implemented with \emph{any} consistent precision estimator. To make the discussion concrete and to match our recommended implementation, we state the complexity under the DRO-calibrated graphical lasso. Under this choice, let $B$ be the number of bootstrap replicates and let $T_{\mathrm{GL}}$ be the number of iterations of the GL solver. Let $d_m$ be the maximum degree of the screened graph induced by the Stage~0 mask $\widehat M$.

\begin{theorem}[Complexity for a fixed ordering]
\label{thm:complexity_fixed_pi}
Fix an ordering $\pi$ and consider the singleton set $\mathbb{T}_p\!=\!\{\pi\}$.
Suppose Stage~0 uses the DRO-calibrated GL.
One pass of Algorithm~\ref{alg:scope} that computes $\widetilde L_\pi$ from $X$ (Stages~0--2) runs in
\[
T(\pi)= O\!\left(B\,n\,p^2 + T_{\mathrm{GL}}\,p^3 + p\,n\,d_m^{\,2}\right).
\]
\end{theorem}
\vspace{-2mm}
Theorem~\ref{thm:complexity_fixed_pi} aggregates the per-stage costs. In the above DRO-GL instantiation, Stage~0 typically dominates in dense-solver worst-case analyses through the $O(T_{\mathrm{GL}}p^3)$ term. Given the mask, IC(0) in Stage~1 contributes $O(p\,d_m^{2})$ via sparse elimination updates on screened neighborhoods. Stage~2 (\textsc{RefitTest}) fits $p$ local regressions with at most $d_m$ predictors, costing $O(p\,n\,d_m^{2})$ when $d_m\!<\!n$. Hence, after Stage~0 the cost scales only linearly in $p$ and is otherwise governed by $d_m$.

Even when Algorithm~\ref{alg:scope} evaluates multiple orderings, Stage~0) is computed once and shared across all $\pi\in\mathbb{T}_p$. Hence, sweeping $\mathbb{T}_p$ costs $|\mathbb{T}_p|$ times the Stage~1--2) work, i.e., $O(|\mathbb{T}_p|\,p\,n\,d_m^{2})$. There can be additional computational costs involved such as computing an AMD ordering from $\widehat{\Omega}$, which costs $O(p\,d_m^{2})$.

\begin{table}[t]
\centering
\caption{Worst-case computational complexity of representative causal structure learning methods.}
\vspace{-1mm}
\label{tab:complexity_summary}
\setlength{\tabcolsep}{1pt} 
\begin{tabular}{l l c}
\hline
Method & Category & Worst-case cost \\
\hline
\textbf{\textsc{SCOPE}} & Sparse Cholesky &
$O\!\left(Bnp^2 + T_{\mathrm{GL}}p^3 + pnd_m^2\right)$ \\
FROSTY & Sparse Cholesky &
$O\!\left(Bnp^2 + T_{\mathrm{GL}}p^3\right)$ \\
RFD & Sparse Cholesky & $O(p^{w+3})$ \\
\hline
PC & Constraint-based & $O(p^{q+2})$ \\
GSP & Constraint-based & $O(p^{q+2})$ \\
\hline
GES & Score-based & $O(p^2 2^p)$ \\
MMHC & Score-based & $O(p^2 2^p)$ \\
\hline
NOTEARS & Continuous opt. & $O(T(n\vee p)p^2)$ \\
GOLEM & Continuous opt. & $O(np^2 + Tp^3)$ \\
\hline
\end{tabular}
\end{table}

Table~\ref{tab:complexity_summary} summarizes worst-case complexity bounds for structure learning baselines: constraint-based methods (PC, GSP)~\cite{kalisch2007pc,solus2021consistency}, score-based methods (GES, MMHC)~\cite{chickering2002ges,tsamardinos2006mmhc}, continuous optimizations (NOTEARS, GOLEM)~\cite{zheng2018notears,ng2020golem}, and sparse-Cholesky pipelines (RFD, FROSTY)~\cite{squires2020rfd,bang2023frosty}. In the table, $q$ is the maximum conditioning-set size for CI-based methods, $w$ is the search-depth parameter for RFD, and $T$ is the number of optimization iterations for NOTEARS and GOLEM. Among the compared methods, \textsc{SCOPE} and FROSTY admit essentially cubic-order worst-case bounds in $p$, up to iteration counts and constant factors, enabling scalability at large $p$.
By contrast, NOTEARS and GOLEM may scale poorly in practice because they repeatedly perform dense matrix computations, often with per-iteration complexity on the order of $p^3$; see Section~\ref{subsec:exp_comparison}.

\begin{figure*}[t]
    \centering
    \includegraphics[width=0.88\linewidth, trim=0pt 0pt 0pt 25pt, clip]{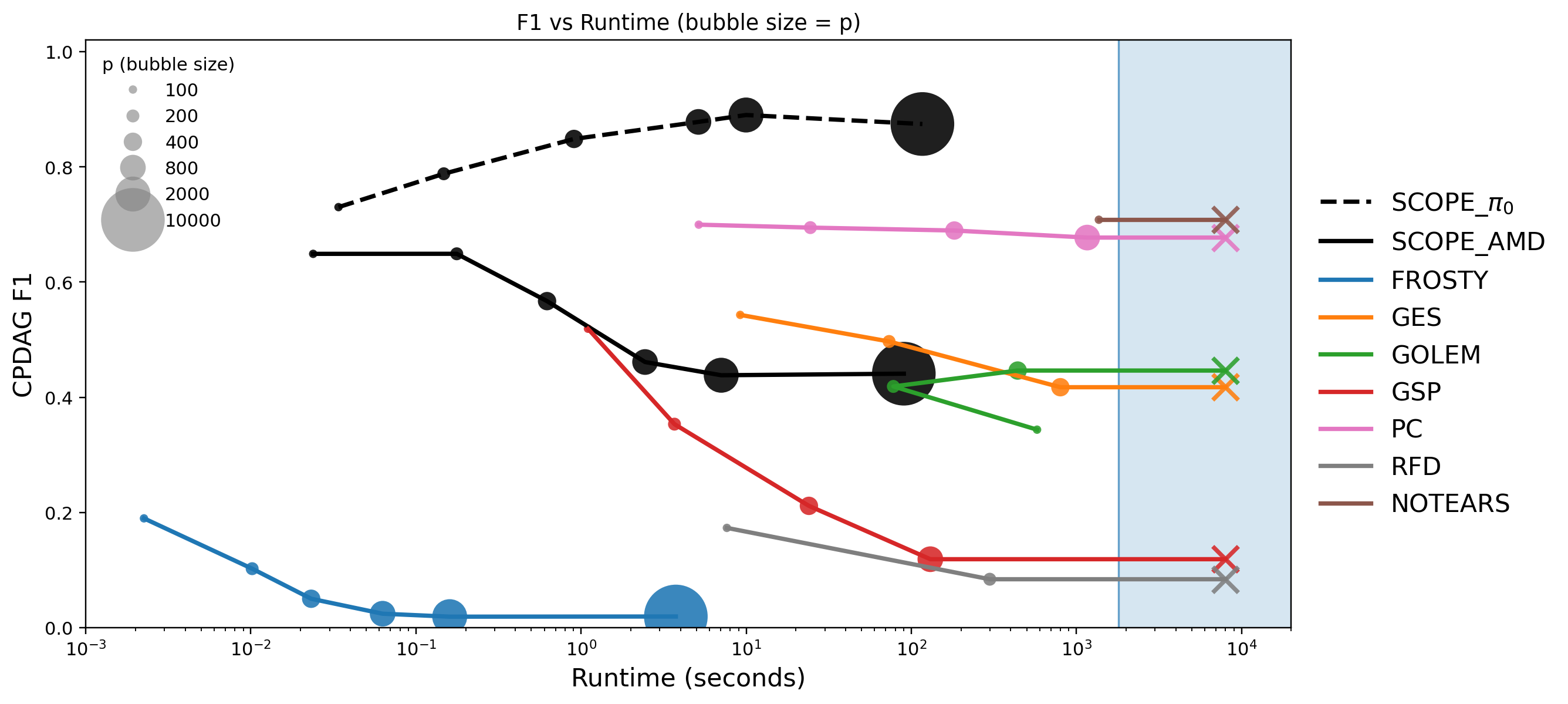}
    \vspace{-2mm}
    \caption{CPDAG $F_1$ versus runtime across problem sizes (bubble size $p$); $\boldsymbol\times$ denotes timed-out runs (shaded region). Two one-pass SCOPE results are shown: when given $\pi_0$ (SCOPE\_$\pi_0$) and when given AMD ordering (SCOPE\_AMD).}
    \vspace{-4mm}
    \label{fig:f1_vs_runtime_bubbles}
\end{figure*}

\vspace{-2mm}
\section{Numerical Experiments}
\label{sec:experiments}
\vspace{-1mm}

In this section, we empirically evaluate our one-pass SCOPE algorithm in finite sample setting where $\widehat{\Omega}$ is estimated from data. Specifically, we evaluate the CPDAG recovery performance when a reverse topological ordering $\pi_0$ or when a heuristic AMD ordering is used.
In Section~\ref{subsec:exp_comparison}, we compare our approach with major causal structure learning methods for various $p$, with $n$ scaled as $20p$ up to a cap of $20,000$.
In Section \ref{subsec:exp_scalability}, one-pass SCOPE algorithm's performance in high-dimensional setting, i.e., $n\leq p$, is explored, highlighting the strength of our unique approach. Additional results and implementation details are given in Appendix~\ref{app:experiments}.


\textbf{Experimental setup.}
Synthetic data used in this section is generated from a linear SEM~\eqref{eq:sem_model} with independent errors.
To reflect the sparsity regime considered throughout the paper, we sample bounded in-degree DAGs where each node selects at most $d$ parents among earlier nodes in an ordering.
We draw nonzero edge weights uniformly from $\pm[0.6,0.8]$ and sample error variances from $[0.8,1.0]$.
We also include non-Gaussian errors to assess robustness beyond Gaussianity.
We report CPDAG-skeleton $F_1$, normalized structural Hamming distance on the CPDAG (nSHD; SHD divided by the number of true edges)~\citep{pmlr-v177-yang22a,ban2024csl_llm}, and wall-clock runtime with a 1800\,s timeout. For precision-based pipelines (\textsc{SCOPE}, FROSTY, RFD), we measure runtime with $\widehat{\Omega}$ given, because precision-estimation time varies widely across solvers and tuning parameter.

\vspace{-2mm}
\subsection{Objective Validation \& Comparison}
\label{subsec:exp_comparison}
\vspace{-1mm}

We vary $p\in\{100,200,400,800,2k,10k\}$ with bounded in-degree $d=5$ under Gaussian noise. 
Figure~\ref{fig:f1_vs_runtime_bubbles} shows that one-pass \textsc{SCOPE} (using estimated $\widehat{\Omega}$) sustains excellent CPDAG recovery as indicated by $F_1$ up to $p\!=\!10k$ (heuristic $\approx0.5$, oracle $\approx0.9$).
Given a correct ordering, one-pass SCOPE (labeled \textsc{SCOPE}\_$\pi_0$) attains the best $F_1$ across all $p$, consistent with the soundness guarantee for the ReSP formulation (Theorem~\ref{thm:guarantees}).
With the AMD ordering, one-pass SCOPE remains competitive (labeled \textsc{SCOPE}\_AMD) despite the decrease in $F_1$, consistent with the robustness guarantee under ordering misspecification (Theorem~\ref{thm:guarantees}).

In the baseline comparison below, we report \textsc{SCOPE} under the AMD heuristic ordering.
Compared with baselines, \textsc{SCOPE} attains a favorable balance of accuracy and scalability.
FROSTY runs in (sub-)second time across all $p$, but its $F_1$ quickly collapses toward zero as $p$ increases.
For GSP, $F_1$ decreases from $0.5$ to $0.1$, and the method fails to return outputs beyond $p\!=\!800$. RFD stays below $0.2$ in $F_1$ and is feasible only up to $p\!=\!200$.
GES and GOLEM are moderately accurate when they finish, with $F_1$ between $0.4$ and $0.5$, but they become infeasible beyond $p\!=\!400$.
NOTEARS completes only at $p\!=\!100$.
PC attains relatively high $F_1$ near $0.7$ at moderate scales but becomes infeasible beyond $p\!=\!800$.
In contrast, \textsc{SCOPE} is the only method that consistently returns results up to $p\!=\!10k$ within the time budget while retaining strong accuracy.



\vspace{-2mm}
\subsection{High-Dimensional Performance}
\label{subsec:exp_scalability}
\vspace{-1mm}

\begin{table}[t]
\centering
\caption{Large-$p$ performance on synthetic linear SEMs.
}
\label{tab:scalability}
\vspace{-3mm}
\begin{subtable}[t]{\linewidth}
\centering
\subcaption{$d=1$}
\vspace{-1mm}
\begin{tabular}{c c c c}
\hline
$n$ & Method & $F_1$ & nSHD \\
\hline
$0.5p$ & SCOPE  & \textbf{0.994} (0.004) & \textbf{0.050} (0.043) \\
$0.5p$ & FROSTY & 0.642 (0.007) & 1.116 (0.034) \\
$p$    & SCOPE  & \textbf{0.951} (0.012) & \textbf{0.264} (0.050) \\
$p$    & FROSTY & 0.635 (0.008) & 1.152 (0.041) \\
\hline
\end{tabular}
\end{subtable}

\vspace{0.6em}

\begin{subtable}[t]{\linewidth}
\centering
\subcaption{$d=5$}
\vspace{-1mm}
\begin{tabular}{c c c c}
\hline
$n$ & Method & $F_1$ & nSHD \\
\hline
$0.5p$ & SCOPE  & \textbf{0.614} (0.117) & \textbf{0.917} (0.046) \\
$0.5p$ & FROSTY & 0.024 (0.001) & 78.041 (4.258) \\
$p$    & SCOPE  & \textbf{0.633} (0.031) & \textbf{1.322} (0.189) \\
$p$    & FROSTY & 0.023 (0.001) & 85.968 (2.966) \\
\hline
\end{tabular}
\end{subtable}

\end{table}

We evaluate high-dimensional performance at $p\!=\!10^4$ with $n\!\in\!\{0.5p,p\}$ under $t$-distributed noise, deferring other noise distributions to Appendix~\ref{app:exp_scalability}.
At this scale, many baselines are computationally infeasible, so we compare \textsc{SCOPE} (AMD heuristic ordering) with FROSTY on bounded in-degree DAG instances; Table~\ref{tab:scalability} reports CPDAG accuracy.

In the zero-fill-friendly regime $d\!=\!1$, moralization does not introduce co-parent links, so sparse elimination incurs essentially no fill~\citep{rose1976elimination,george1981computer}.
Accordingly, \textsc{SCOPE} remains near-exact at $p\!=\!10^4$ under the AMD heuristic ordering, achieving $F_1\!>\!0.95$ with small nSHD.
In contrast, FROSTY attains substantially lower $F_1$ (about $0.64$) and noticeably larger nSHD.
The gap widens at $d\!=\!5$, where co-parent effects densify the precision-support mask.
\textsc{SCOPE} maintains moderate accuracy ($F_1\!>\!0.61$) with nSHD around $1$, whereas FROSTY exhibits near-zero $F_1$ and extremely large nSHD.
Overall, these results show that \textsc{SCOPE} continues to yield usable CPDAG estimates in high dimensions, even with heavy-tailed noise.

\vspace{-4mm}
\section{Breast Cancer Data Analysis}
\label{sec:real}
\vspace{-1mm}

To demonstrate real-world performance, we analyze an mRNA expression dataset of invasive ductal breast carcinoma samples from The Cancer Genome Atlas (TCGA) \citep{TheCancerGenomeAtlasNetwork2012}, processed as in \cite{Koh2019}.
Because transcription factors (TFs) regulate the expression of target genes, inferred edges between them are expected to be directed from TF to target.
We evaluate the efficacy of SCOPE with the AMD ordering in inferring the CPDAG by comparing our directional discoveries with the biologically validated TF regulatory network from \cite{Koh2019}.

We analyze a large TCGA mRNA expression dataset with $(n,p)=(599,16325)$ and apply \textsc{SCOPE} on the full, unpartitioned dataset. Due to the limited scalability of the baseline methods, we evaluate PC and GES using a split-and-union protocol based on random disjoint feature partitions. Specifically, we run PC on subsets of $\sim$800 features (GES on subsets of $\sim$400 features), and subsequently take the union of the directed edges across all respective splits. The choices of these baselines and their partition sizes are based on the runtime limits established in Figure~\ref{fig:f1_vs_runtime_bubbles}. Additional results and implementation details are provided in Appendix~\ref{app:real}.

\begin{table}[t]
\centering
\caption{Proportions of biologically validated TF$\to$target path recovered in the estimated CPDAGs.
}
\vspace{-1mm}
\label{tab:tcga_tfdb_hops_compact}
\begin{tabular}{l c c c c}
\toprule
& \multicolumn{2}{c}{One-step Paths} & \multicolumn{2}{c}{Two-step Paths} \\
\cmidrule(lr){2-3} \cmidrule(lr){4-5}
Method & Fwd/Rev & Ratio & Fwd/Rev & Ratio \\
\midrule
PC & 0.028/0.024 & 1.167 & 0.051/0.045 & 1.133 \\
GES & 0.028/0.024 & 1.167 & 0.056/0.045 & 1.244 \\
\textsc{SCOPE} & 0.078/0.045 & \textbf{1.733} & 0.108/0.061 & \textbf{1.770} \\
\bottomrule
\end{tabular}
\end{table}

Our main performance metric is whether edge directions are identified correctly versus in error. Specifically, we calculate the forward and reverse edge identification rates in estimated CPDAGs. The \emph{forward} rate (Fwd) is the number of validated TF$\to$target pairs supported by a directed path from TF to target, normalized by the number of directed edges. One-step (direct TF$\to$target edge) and two-step (directional TF$\to$target path with one intermediate mediator) paths are considered. The \emph{reverse} rate (Rev) counts the same pairs with direction flipped. 
Note that while the reversed edges misidentify the true causal direction of the regulatory relationship, their presence in the external database still indicates a biologically meaningful association.

Table~\ref{tab:tcga_tfdb_hops_compact} summarizes these two evaluation metrics and their ratios.
\textsc{SCOPE} produces 271,066 directed edges, whereas PC and GES produce 62,422 and 308,533, respectively.
\textsc{SCOPE} achieves substantially higher forward rates than PC and GES and exhibits the largest forward-to-reverse ratios. \textsc{SCOPE} also has the highest proportion of externally-validated associations (regardless of direction): 0.292 of its directed paths are biologically validated, compared with 0.148 and 0.153 for PC and GES, respectively. Despite operating on the full dataset, \textsc{SCOPE} completes CPDAG estimation in about 2 hours, versus roughly 5 hours for PC and 20 hours for GES under the split-and-union protocol.

Although random feature splitting enables CPDAG estimation using scalability-limited methods, PC and GES, the procedure can potentially have adverse impacts. A plausible explanation is that random feature splitting can separate TFs from their targets, rendering the recovery of regulatory evidence impossible.
Moreover, split-wise marginalization of many genes can distort the CI relations used for orientation and introduce confounding-like effects. Combined with results in Table~\ref{tab:tcga_tfdb_hops_compact}, these observations highlight the practical value of methods such as \textsc{SCOPE} that operate on the full set of features, thereby preserving direct evidence and improving the credibility of inferred directions.

\vspace{-3mm}
\section{Discussion}
\vspace{-1mm}

This work advances MEC recovery to large graphs by reframing SP learning, thereby alleviating the combinatorial search over variable orderings. Rather than enforcing the exact Cholesky constraint $\Omega_\pi = L_\pi L_\pi^\top$, we show that MEC recovery is possible via the ReSP formulation. ReSP replaces exact factorization of a full matrix with support-level indicated entries and thereby loosens the tight coupling between variable ordering and numerical factorization that affects structural sparsity.
\textsc{SCOPE} operationalizes this relaxation through a scalable sparse-Cholesky pipeline for ordering comparison. At the population level, the relaxation preserves the identifiability guarantees of SP under suitable conditions and also provides an intrinsic robustness property. Even under ordering misspecification, feasibility under the precision-support mask confines induced edges to the moralized skeleton. Empirically, ReSP and its implementation \textsc{SCOPE} remain competitive at a scale of $10^4$ variables in simulations and excels in analyzing real mRNA expression data as validated with an external database.

The ReSP formulation suggests several extensions. A first direction is to develop local ordering search tailored to ReSP, for example via swap- and tuck-type moves.
Second, the modularity of \textsc{SCOPE} makes parallel and distributed variants natural, such as multi-ordering evaluation and parallelized ordering heuristics, with the potential to improve scalability further.
Finally, enforcing feasibility through the mask rather than an exact factorization identity may provide a useful starting point beyond well-specified linear SEMs.
Specifically, it can be adapted to computationally challenging misspecified regimes, such as interventional and latent-variable settings~\citep{squires20a,squiresLatentClusters}.

\begin{contributions} 
    Briefly list author contributions. 
    This is a nice way of making clear who did what and to give proper credit.
    This section is optional.

    H.~Q.~Bovik conceived the idea and wrote the paper.
    Coauthor One created the code.
    Coauthor Two created the figures.
\end{contributions}

\begin{acknowledgements} 
    Authors thank Professor Hyung Won Choi at National University of Singapore for helpful advice on real data analysis.
\end{acknowledgements}

\bibliography{uai2026-ReSP}

\newpage

\onecolumn

\title{Relaxed Sparsest-Permutation Formulation for Causal Discovery at Scale (Appendix)}
\maketitle

\appendix

\section{A mismatch between fill reduction and SP sparsity}
\label{app:counterexample}

This appendix provides an illustrative example showing that minimizing fill in the undirected elimination graph does not necessarily yield a sparse triangular factor under the SP criterion.
It also highlights a subtle point: standard fill-reducing approaches optimize an elimination-graph objective that need not match the realized support of the numerical Cholesky factor, although the two are often aligned in practice for sparse Cholesky factorization.

\begin{figure}[t]
  \centering

  \begin{subfigure}[t]{0.19\textwidth}
    \centering
    \includegraphics[width=\linewidth]{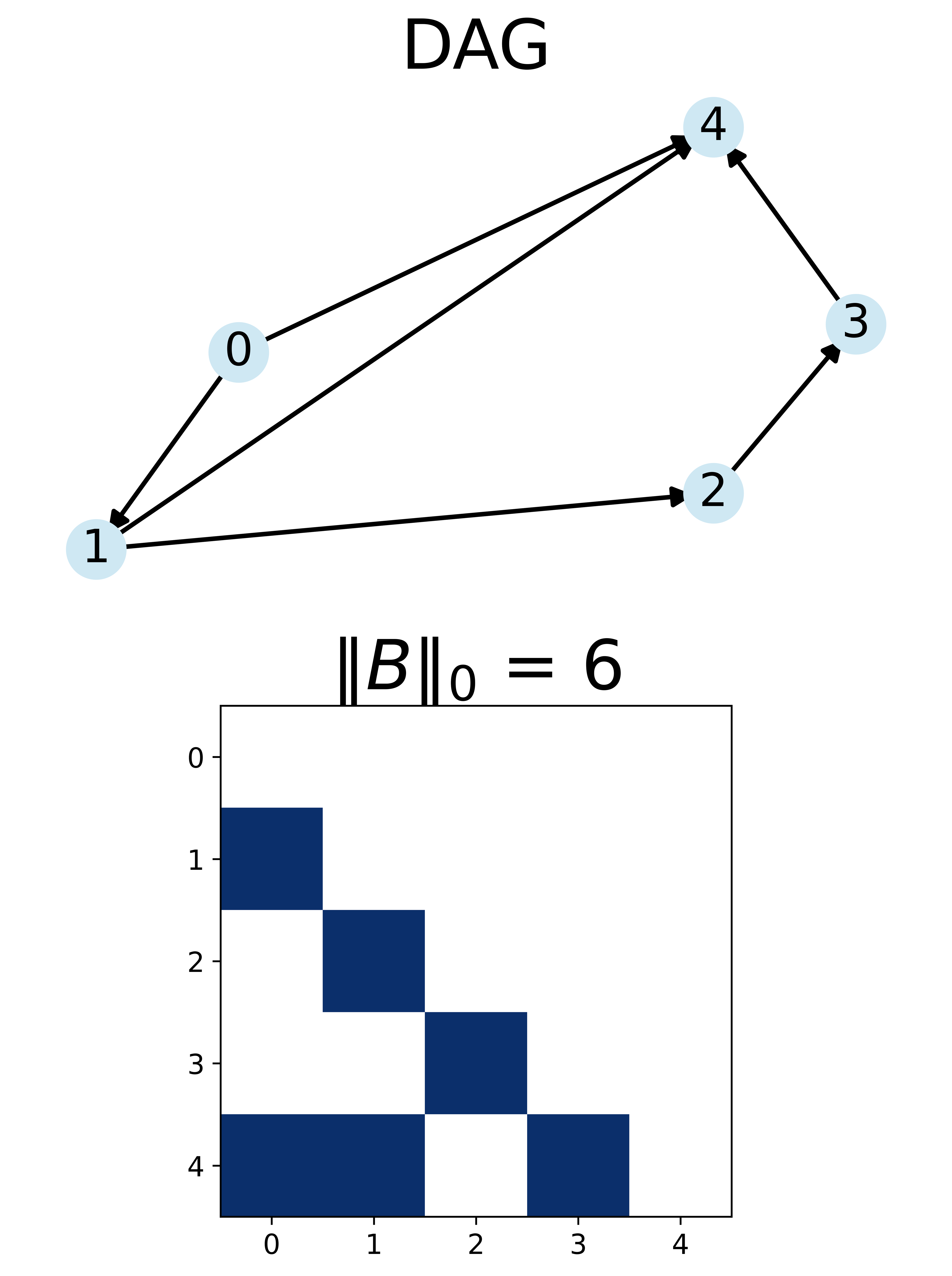}
    \caption{DAG}
    \label{fig:counterexample:dag}
  \end{subfigure}\hfill
  \begin{subfigure}[t]{0.19\textwidth}
    \centering
    \includegraphics[width=\linewidth]{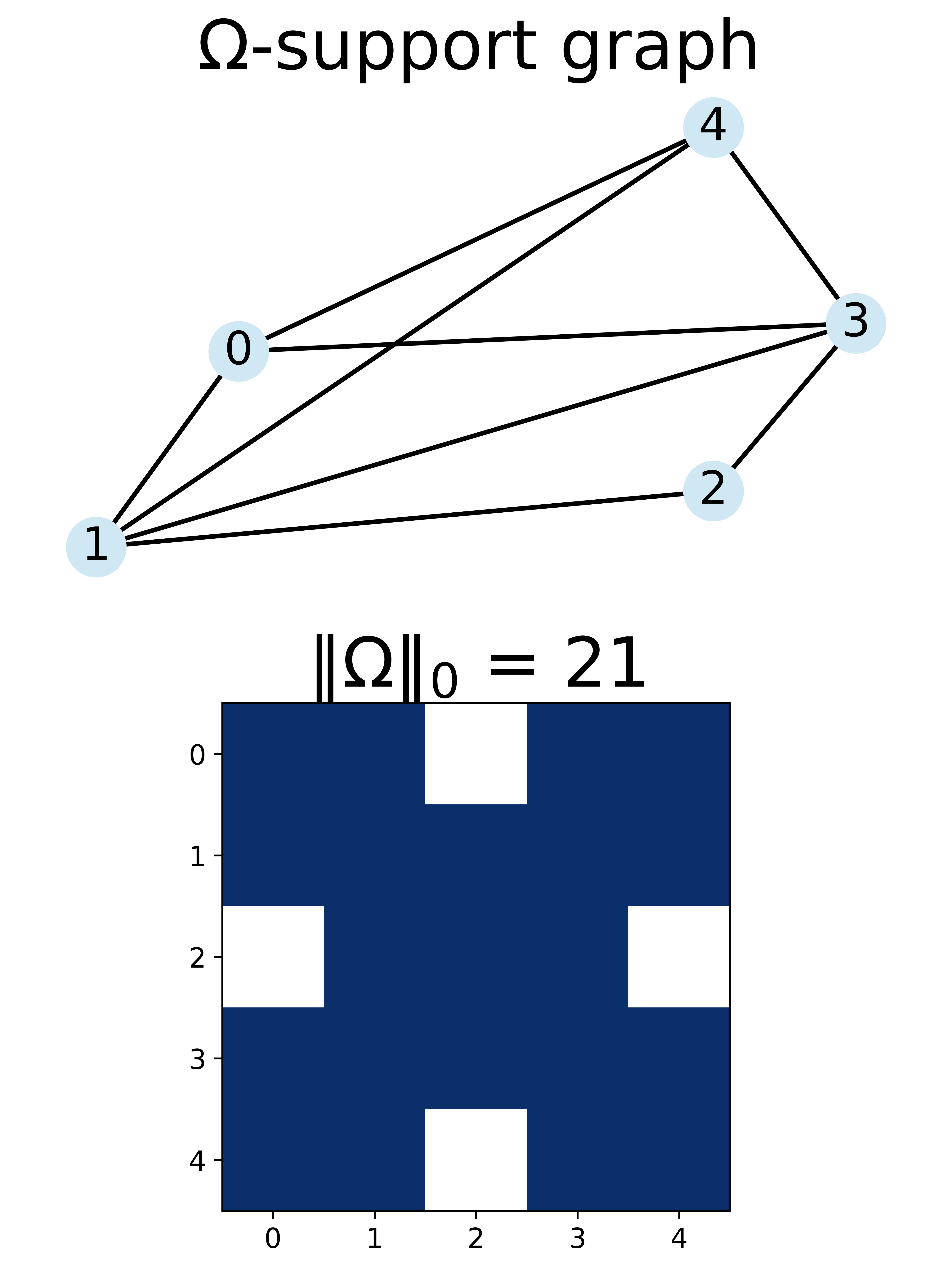}
    \caption{Moral graph}
    \label{fig:counterexample:moral}
  \end{subfigure}\hfill
  \begin{subfigure}[t]{0.19\textwidth}
    \centering
    \includegraphics[width=\linewidth]{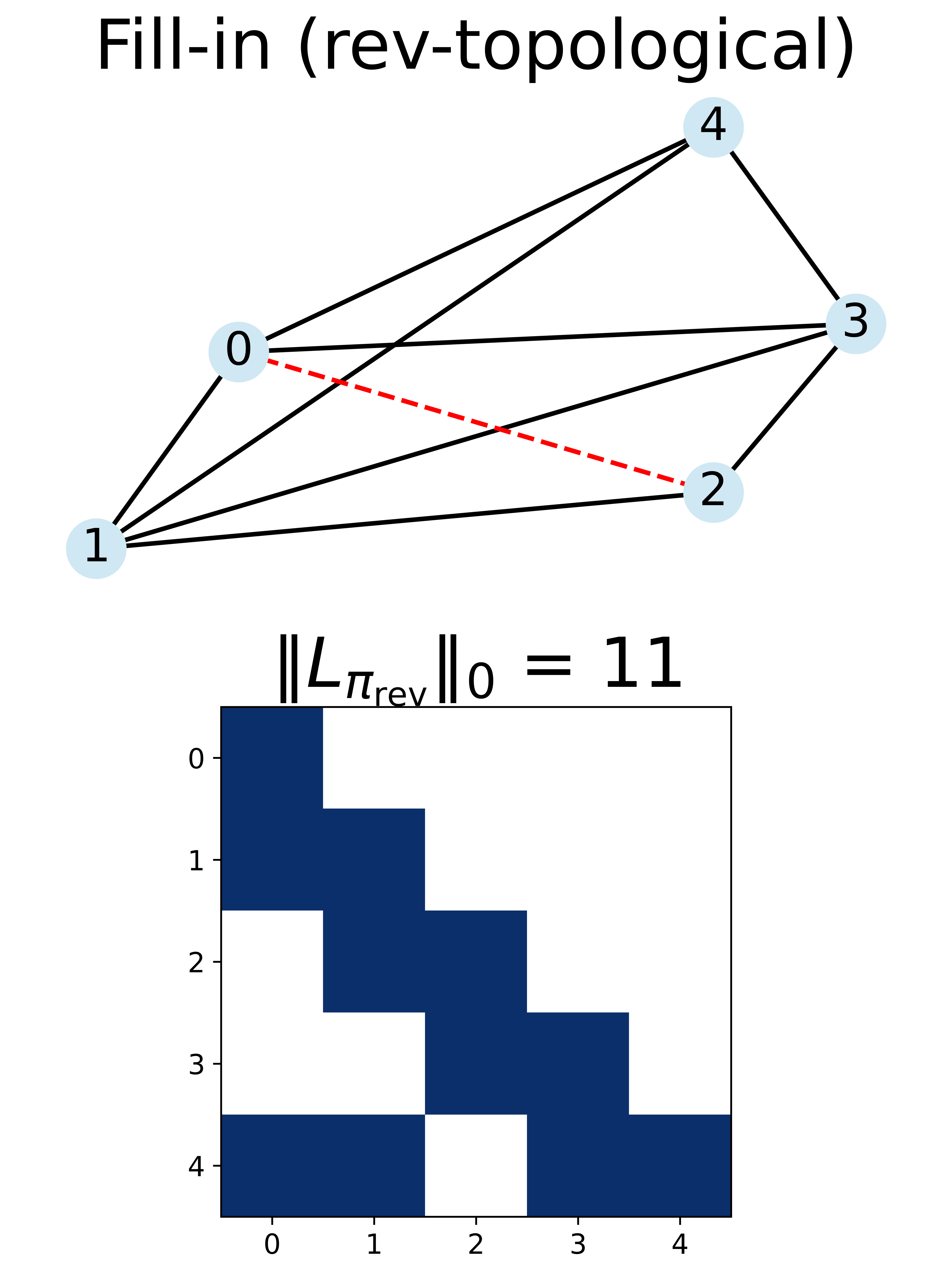}
    \caption{rev-topo}
    \label{fig:counterexample:rev}
  \end{subfigure}\hfill
  \begin{subfigure}[t]{0.19\textwidth}
    \centering
    \includegraphics[width=\linewidth]{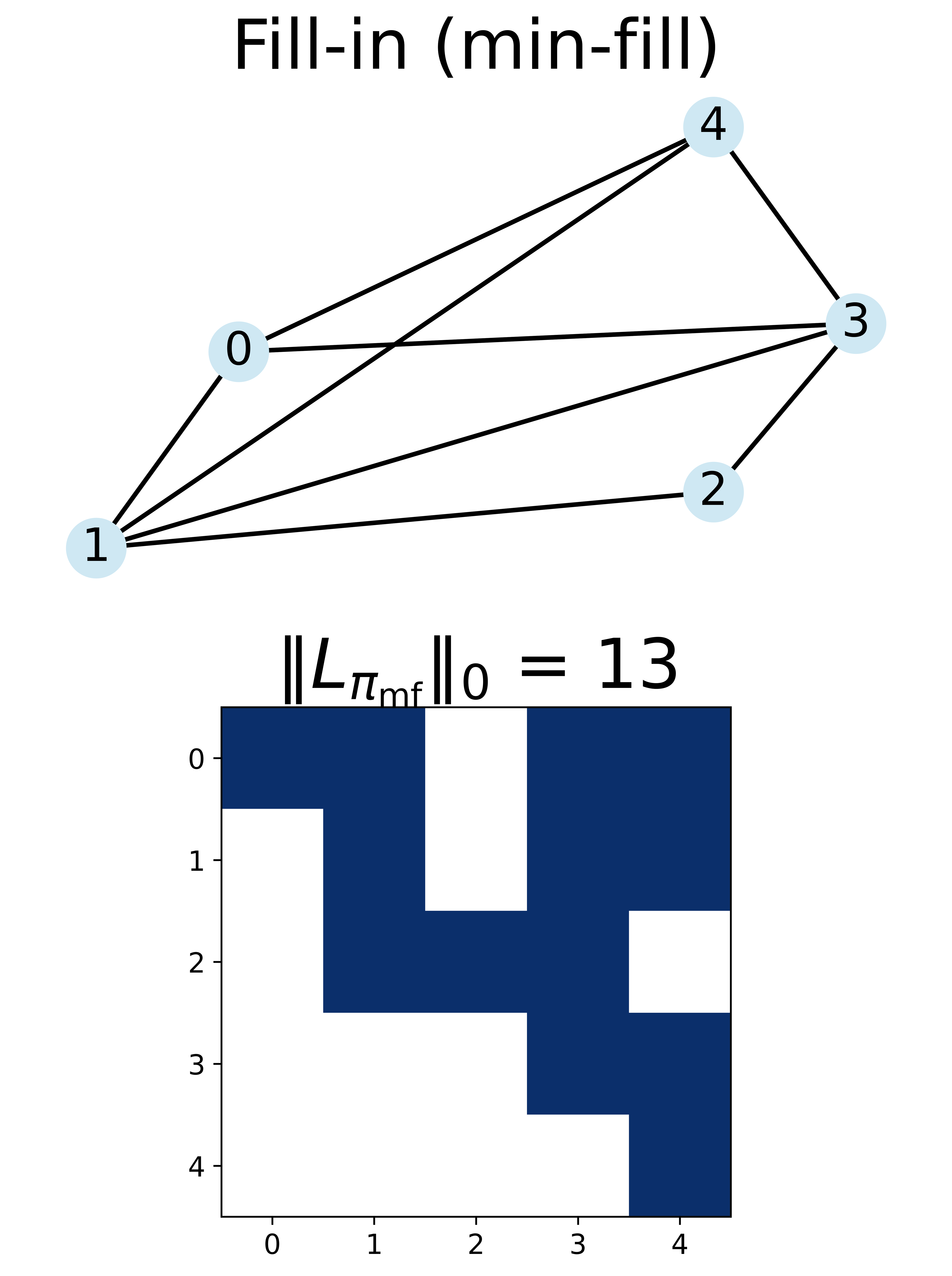}
    \caption{min-fill}
    \label{fig:counterexample:minfill}
  \end{subfigure}\hfill
  \begin{subfigure}[t]{0.19\textwidth}
    \centering
    \includegraphics[width=\linewidth]{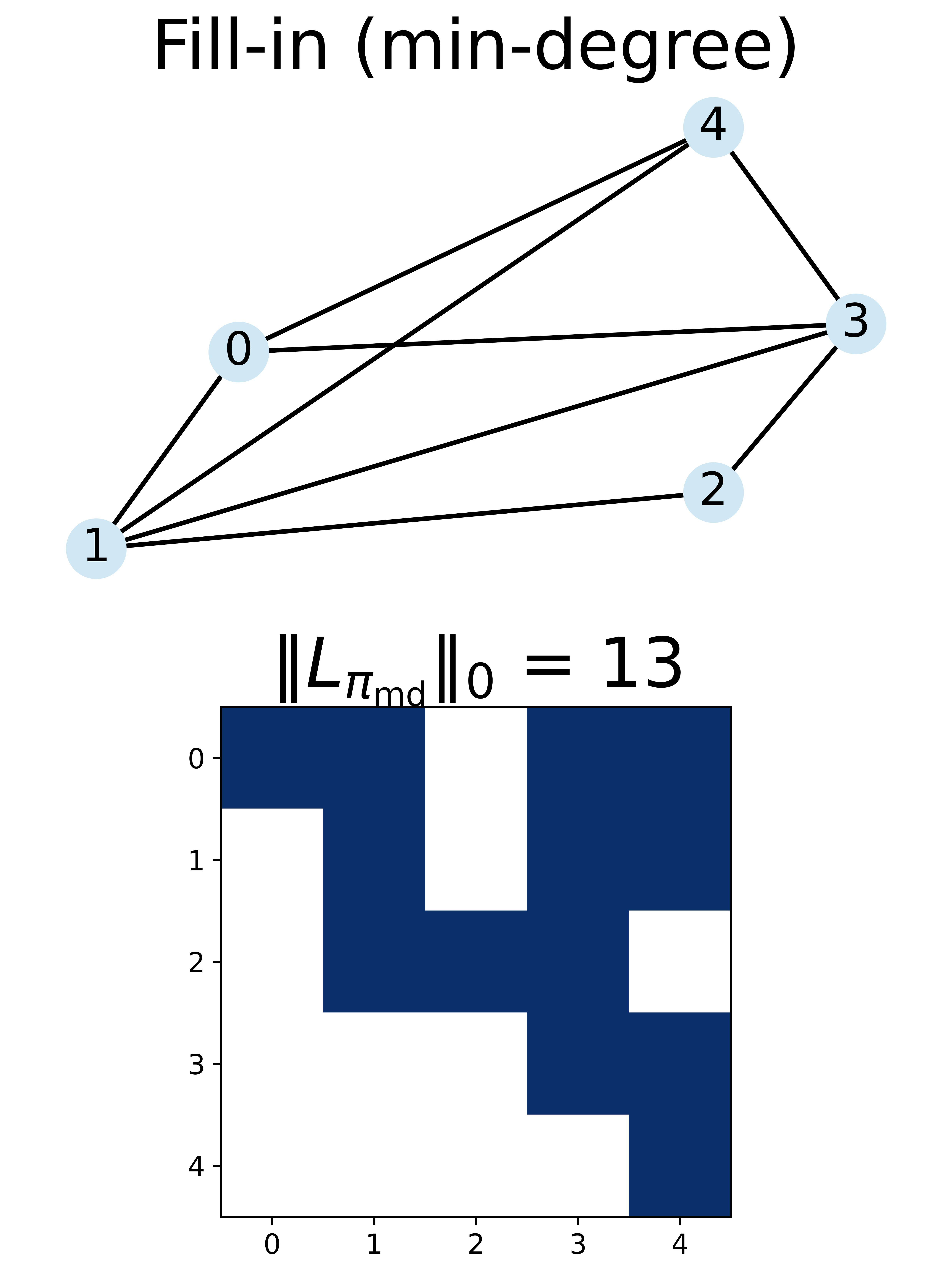}
    \caption{min-degree}
    \label{fig:counterexample:mindeg}
  \end{subfigure}

  \caption{DAG, moral graph, and triangulations under rev-topo, min-fill, and min-degree orderings, with sparsity patterns of $B$, $\Omega$, and $L$. Red dashed edges indicate fill-in.}
  \label{fig:counterexample}
\end{figure}

Figure~\ref{fig:counterexample} illustrates that fill-reducing approaches based on the undirected elimination graph can disagree with the SP objective, and that $\|L_\pi\|_0$ can be highly sensitive to the chosen ordering.
The DAG and the zero pattern of the adjacency matrix $B$ are shown in Figure~\ref{fig:counterexample}\subref{fig:counterexample:dag}.
Figure~\ref{fig:counterexample}\subref{fig:counterexample:moral} shows the moralized graph and the corresponding precision support $\supp(\Omega)$.
In particular, the common child at node $4$ induces moral edges between nodes $(0,3)$ and $(1,3)$, which are absent in the original DAG skeleton but appear in the moralized graph.
Figures~\ref{fig:counterexample}\subref{fig:counterexample:rev}--\subref{fig:counterexample:mindeg} compare three representative elimination orderings.
The reverse topological ordering is $\pi_{\mathrm{rev}}=[4,3,2,1,0]$ in Figure~\ref{fig:counterexample}\subref{fig:counterexample:rev}.
The minimum-fill ordering is $\pi_{\mathrm{mf}}=[0,2,1,3,4]$ in Figure~\ref{fig:counterexample}\subref{fig:counterexample:minfill}.
The minimum-degree ordering is $\pi_{\mathrm{md}}=[2,0,1,3,4]$ in Figure~\ref{fig:counterexample}\subref{fig:counterexample:mindeg}.

Figure~\ref{fig:counterexample}\subref{fig:counterexample:rev} shows that the reverse ordering introduces one fill edge $(0,2)$ in the undirected elimination graph, yet it yields the sparsest Cholesky factor with $\|L_{\pi_{\mathrm{rev}}}\|_0=11$.
In contrast, Figures~\ref{fig:counterexample}\subref{fig:counterexample:minfill} and~\subref{fig:counterexample:mindeg} achieve zero fill in the undirected elimination graph so that $|\mathcal{F}(\pi)|=0$, but both produce denser triangular factors with $\|L_{\pi_{\mathrm{mf}}}\|_0=\|L_{\pi_{\mathrm{md}}}\|_0=13$.
Thus, eliminating fill in the undirected elimination graph does not guarantee sparsity of $L_\pi$ under the SP criterion.

The mechanism can be seen by tracking what happens when the elimination order is aligned with the true directionality.
Under $\pi_{\mathrm{rev}}$, node $4$ is removed first.
When node $3$ is processed next, the undirected elimination graph treats the moral edges $(0,3)$ and $(1,3)$, created because $0$, $1$, and $3$ share the common child $4$, as still present.
It therefore predicts additional interactions among $\{0,1,2\}$ at the step for node 3 and produces the fill edge $(0,2)$.
Numerically, once node $4$ is removed, the coupling between $3$ and $\{0,1\}$ that is mediated only through the common child $4$ vanishes, so the moral edges $(0,3)$ and $(1,3)$ do not generate subsequent nonzeros in the Cholesky factor.
In this sense, fill predicted from the undirected elimination graph can overestimate the nonzeros that materialize in the numerical factorization.
This gap explains why $\pi_{\mathrm{rev}}$ incurs nonzero fill in the undirected elimination graph while still producing a sparser $L_\pi$.

More importantly, when the elimination order is misaligned with the true ordering, the numerical factor can become markedly denser.
This inflation is typically even more pronounced when the ordering also fails to reduce fill in the undirected elimination graph.
The effect is visible in Figures~\ref{fig:counterexample}\subref{fig:counterexample:minfill} and~\ref{fig:counterexample}\subref{fig:counterexample:mindeg}, where the triangular patterns contain many additional nonzeros, including in upper-triangular positions of the displayed sparsity maps.
Intuitively, under a misspecified ordering, the induced directed edges disagree with the true causal directions, which increases the number of nonzeros required to represent the same second-order structure in triangular form.
Consequently, using fill reduction alone as a surrogate for SP sparsity can be brittle for directed structure recovery.

Overall, this counterexample shows that controlling fill in the undirected elimination graph alone can be misleading for SP learning, which motivates imposing an ordering-agnostic screening such as restricting feasibility to $\supp(\Omega)$ before making sparsity comparisons across permutations.

\section{Finite-Sample Implementation of Stage 0}
\label{app:sample}

This section details Stage~0 of \textsc{SCOPE}: finite-sample precision estimation and precision-support mask construction.
We estimate the precision matrix using the graphical lasso and select its regularization parameter $\lambda$ via a distributionally robust optimization (DRO) calibration \citep{friedman2008glasso, cisnerosvelarde2020droglasso}.
We then apply $\lambda$-scale hard-thresholding, with the threshold scale justified by standard entrywise error bounds for precision estimation \citep{ravikumar2011logdet,liu2012copula}.
The support of the thresholded precision estimate then defines the precision-support mask.

\paragraph{Graphical lasso with DRO calibration.}
Let $\mathrm{Sym}_p^{+}$ and $\mathrm{Sym}_p^{++}$ denote the cones of symmetric positive semidefinite
and symmetric positive definite matrices, respectively.
Let $W_i:=X_iX_i^\top\in\mathrm{Sym}_p^{+}$ and
$\widehat\Sigma:=\frac{1}{n}\sum_{i=1}^n W_i\in\mathrm{Sym}_p^{+}$ be the empirical covariance matrix.
The graphical lasso estimates the precision matrix by
\begin{equation}
  \label{eq:GL}
  \min_{K\in \mathrm{Sym}_p^{++}}
  \Bigl\{
    \mathrm{tr}(K\widehat\Sigma)
    - \log\det(K)
    + \lambda\|K\|_{1,\mathrm{off}}
  \Bigr\},
\end{equation}
which corresponds to $\ell_1$-penalized Gaussian likelihood maximization \citep{friedman2008glasso}.
Following \citep[Theorem~2.1]{cisnerosvelarde2020droglasso}, the same objective also arises from a worst-case
expected Gaussian log-likelihood over a Wasserstein-type ambiguity set.
We denote by $\widehat\Omega(\lambda)$ any solution to~\eqref{eq:GL}.
It remains to choose $\lambda$ in a way that yields explicit high-probability error bounds for $\widehat\Omega(\lambda)$.

\paragraph{Entrywise deviation control and the choice of $\lambda$.}
Under the DRO interpretation of \citet{cisnerosvelarde2020droglasso}
(rooted in \citet{blanchet2019robust}), the regularization parameter $\lambda$ is naturally tied to an
entrywise deviation scale for the sample covariance.
This choice is also consistent with classical finite-sample analyses of high-dimensional precision estimation,
which typically condition on an event controlling $\|\widehat\Sigma-\Sigma\|_\infty$.
For $\alpha\in(0,1)$, define
\begin{equation}
  \label{eq:lambda_dro_inf}
  \lambda
  :=
  \inf\Bigl\{
    t>0:
    \mathbb{P}_0\!\bigl(\|\widehat\Sigma-\Sigma\|_\infty \le t\bigr)
    \ge 1-\alpha
  \Bigr\},
\end{equation}
and the associated event
\begin{equation}
\label{eq:Ealpha}
\mathcal{E}_\alpha
:=
\Bigl\{\,
\|\widehat\Sigma-\Sigma\|_\infty \le \lambda
\,\Bigr\}.
\end{equation}
In practice, $\lambda$ in~\eqref{eq:lambda_dro_inf} is estimated via the bootstrap RobSel procedure
\citep{cisnerosvelarde2020droglasso}.
Consistency of this bootstrap choice is established in low-dimensional regimes
\citep[Remark~3.4--3.5]{cisnerosvelarde2020droglasso} and \citep{tran2022family}.
In our setting, we use RobSel as a data-adaptive way to set $\lambda$ for entrywise control.

\paragraph{Entrywise bounds on $\widehat\Omega(\lambda)$ and thresholding justification.}
Under standard conditions for graphical lasso, an event like~\eqref{eq:Ealpha} is the natural input for controlling
the entrywise error of the resulting precision estimate.
In particular, under incoherence-type assumptions and bounded-degree conditions, the graphical lasso solution satisfies
\[
\|\widehat\Omega(\lambda)-\Omega\|_\infty \le C\lambda
\quad \text{on}\quad\mathcal{E}_\alpha,
\]
for a constant $C$ depending on $\Sigma$ and the model class \citep{ravikumar2011logdet}.
If, moreover, the nonzero off-diagonal entries of $\Omega$ satisfy a minimum signal strength condition
\[
\theta_{\min}:=\min_{(i,j)\in S}|\Omega_{ij}| \ge C'\lambda,
\qquad
S:=\{(i,j): i\neq j,\ \Omega_{ij}\neq 0\},
\]
then hard-thresholding at the same $\lambda$-scale recovers the off-diagonal support and signs of $\Omega$ on $\mathcal{E}_\alpha$; see \citet{ravikumar2011logdet} and \citet[Theorem~4.3]{liu2012copula}.
Concretely, for a constant $c_{\mathrm{thr}}>0$, define
\begin{equation}
\label{eq:Omega_thr}
\widehat{\Omega}^{\mathrm{thr}}_{ij}
:=
\widehat{\Omega}(\lambda)_{ij}\,
\mathbf{1}\bigl\{|\widehat{\Omega}(\lambda)_{ij}|\ge c_{\mathrm{thr}}\lambda\bigr\},
\quad (i\neq j),
\qquad
\widehat{\Omega}^{\mathrm{thr}}_{ii}:=\widehat{\Omega}(\lambda)_{ii}.
\end{equation}

\paragraph{Precision-support mask used in \textsc{SCOPE}.}
We use $\widehat\Omega^{\mathrm{thr}}$ with $c_{\mathrm{thr}}=1$ as the input precision matrix and construct the binary mask
\[
\widehat{M} := \mathbf{1}\{\widehat{\Omega}^{\mathrm{thr}} \neq 0\},
\]
with the diagonal included.
We take $\widehat{M}$ as an ordering-agnostic screen of admissible interactions.
When the above entrywise control and the minimum signal strength condition hold, the mask screens out locations where $\Omega$ is truly zero while
retaining the true support with high probability \citep{ravikumar2011logdet,liu2012copula}.
This provides theoretical evidence for using a precision-support mask as an ordering-agnostic admissible pattern for the subsequent
masked IC(0) factorization.

\section{Proofs of Theoretical results}
\label{app:proofs}

\subsection{Proof of Lemma~\ref{lem:mask_validity}}
\label{app:proof_mask_validity}

By \citet[Theorem~2 and Assumption~1]{loh2014icov}, under Assumption~\ref{assm:nocancel} we have, for all $j\neq k$,
\[
\Omega_{jk}=0
\quad\Longleftrightarrow\quad
(j,k)\notin \skel(G_0^m).
\]
Equivalently,
\[
\supp(\Omega)=\skel(G_0^m).
\]

Next, for any permutation matrix $P_\pi$,
\[
\Omega_\pi \;=\; P_\pi\,\Omega\,P_\pi^\top
\]
is a relabeling of indices, and hence
\[
\supp(\Omega_\pi)
\;=\;
\pi\bigl(\supp(\Omega)\bigr)
\;=\;
\pi\bigl(\skel(G_0^m)\bigr).
\]

Finally, since $\skel(G_0)\subseteq \skel(G_0^m)$, applying $\pi$ gives
\[
\pi\bigl(\skel(G_0)\bigr)
\;\subseteq\;
\pi\bigl(\skel(G_0^m)\bigr)
\;=\;
\supp(\Omega_\pi),
\]
which is the claim.

\subsection{Proof of Theorem~\ref{thm:guarantees}}
\label{app:proof_guarantees}

\subsubsection{Proof for Soundness}
\label{app:proof_soundness}
We first record a simple consequence of precision-support masking and IC(0)-admissibility.
\begin{lemma}[Exact factorization under precision-support masking and IC(0)-admissibility]
\label{lem:exact_factorization_from_mask}
    Assume $M=\mathbf{1}\{\Omega\neq 0\}$ (with the diagonal included) and let $\mathbb{T}_p$ be IC(0)-admissible.
    Then, for every $\pi\in\mathbb{T}_p$, the masked IC(0) output $L_\pi=\mathrm{IC0}(\Omega_\pi;M_\pi)$ satisfies 
    \[ \Omega_\pi = L_\pi L_\pi^\top, \qquad 
    \text{where}\quad \Omega_\pi:=P_\pi\Omega P_\pi^\top,\; M_\pi:=P_\pi M P_\pi^\top. 
    \]
\end{lemma}
\begin{proof} 
Fix $\pi\in\mathbb{T}_p$.
Since $M=\mathbf{1}\{\Omega\neq 0\}$, we have 
$\Omega\circ(\mathbf 1-M)=\mathbf 0$, 
and relabeling yields 
\[ \Omega_\pi\circ(\mathbf 1-M_\pi)=\mathbf 0. \] 
By the masked IC(0) invariants~\eqref{eq:masked_ic0_constraints}, 
\[ (\Omega_\pi-L_\pi L_\pi^\top)\circ M_\pi=\mathbf 0. \] 
Moreover, IC(0)-admissibility~\eqref{eq:admissible} gives 
\[ (L_\pi L_\pi^\top)\circ(\mathbf 1-M_\pi)=\mathbf 0. \] 
Combining the last two equations and splitting over $M_\pi$ and $(\mathbf 1-M_\pi)$, 
\[ \Omega_\pi-L_\pi L_\pi^\top 
= \bigl[(\Omega_\pi-L_\pi L_\pi^\top)\circ M_\pi\bigr] + \bigl[(\Omega_\pi-L_\pi L_\pi^\top)\circ(\mathbf 1-M_\pi)\bigr] 
=\mathbf 0.
\] 
Here, the second term vanishes because 
\[ (\Omega_\pi-L_\pi L_\pi^\top)\circ(\mathbf 1-M_\pi) 
= \Omega_\pi\circ(\mathbf 1-M_\pi)
-(L_\pi L_\pi^\top)\circ(\mathbf 1-M_\pi) 
=\mathbf 0-\mathbf 0. 
\] 
Hence $\Omega_\pi=L_\pi L_\pi^\top$. 
\end{proof} 

Lemma~\ref{lem:exact_factorization_from_mask} implies that for every $\pi\in\mathbb{T}_p$,
the constraints in~\eqref{eq:new_objective} reduce to the exact factorization constraint
\[
L_\pi L_\pi^\top=\Omega_\pi.
\]
Hence~\eqref{eq:new_objective} is equivalent to minimizing the sparsity of the exact Cholesky factor of $\Omega_\pi$
over candidate orderings:
\[
\min_{\pi \in \mathbb{T}_{p}} \ \|L_\pi\|_0
\quad \text{s.t.}\quad
L_\pi L_\pi^\top=\Omega_\pi.
\]
This is the same sparsity-based comparison underlying SP learning in \citet[Def.~2.2 and Thm.~2.3]{raskutti2018sp} so the remaining steps follow the standard SMR argument.

Let $(\hat\pi,L_{\hat\pi})$ be any minimizer of~\eqref{eq:new_objective} over $\mathbb{T}_p$.
By feasibility of $(\pi_0,L_{\pi_0})$ and optimality,
\begin{equation}
\label{eq:basic_opt_ineq}
\|L_{\hat\pi}\|_0 \le \|L_{\pi_0}\|_0.
\end{equation}
If $G_{\hat\pi}\notin\mathcal{M}(G_0)$, then Assumption~\ref{assm:SMR2} implies
\[
\|L_{\hat\pi}\|_0 > \|L_{\pi_0}\|_0,
\]
which contradicts~\eqref{eq:basic_opt_ineq}. Therefore $G_{\hat\pi}\in\mathcal{M}(G_0)$.

\subsubsection{Proof of Robustness}
\label{app:proof_robustness}

Fix an arbitrary ordering $\pi$, and let $L_\pi=\mathrm{IC0}(\Omega_\pi;M_\pi)$ be the lower-triangular matrix satisfying the masked feasibility constraints in~\eqref{eq:masked_ic0_constraints} under $\pi$.
Recall the population mask $M := \mathbf{1}\{\Omega \neq 0\}$ and its permuted version
$M_\pi := P_\pi M P_\pi^\top = \mathbf{1}\{\Omega_\pi \neq 0\}$, where $\Omega_\pi := P_\pi \Omega P_\pi^\top$.

Since $L_\pi$ is feasible for~\eqref{eq:masked_ic0_constraints}, it satisfies the support constraint
\begin{equation}\label{eq:feas_mask_app}
L_\pi \circ (1-M_\pi)=0.
\end{equation}
Consequently, for any $k>j$,
\[
(L_\pi)_{kj}\neq 0 \;\Longrightarrow\; (M_\pi)_{kj}=1 \;\Longrightarrow\; (\Omega_\pi)_{kj}\neq 0.
\]
Since $(\Omega_\pi)_{kj}=\Omega_{\pi(k),\pi(j)}$ and $\Omega$ is symmetric, the condition $(\Omega_\pi)_{kj}\neq 0$ is equivalent to both $\Omega_{\pi(k),\pi(j)}\neq 0$ and $\Omega_{\pi(j),\pi(k)}\neq 0$, i.e.,
\[
(L_\pi)_{kj}\neq 0 \;\Longrightarrow\; (\pi(k),\pi(j)),(\pi(j),\pi(k))\in \supp(\Omega).
\]
By Assumption~\ref{assm:nocancel}, $\supp(\Omega)=\skel(G_0^m)$. Therefore every directed edge induced by a nonzero entry of $L_\pi$ connects a pair adjacent in $\skel(G_0^m)$, and hence
\[
\widetilde{E}(\pi,L_\pi)
:= \bigl\{\,(\pi(j),\pi(k)):\ k>j,\ (L_\pi)_{kj}\neq 0 \,\bigr\}
\ \subseteq\ \skel(G_0^m),
\]
which proves the claim.

\subsection{Proof of Theorem~\ref{thm:complexity_fixed_pi}}
\label{app:proof_complexity_fixed_pi}

We prove Theorem~\ref{thm:complexity_fixed_pi} by decomposing the runtime across stages.

Fix an ordering $\pi$ and consider one pass of Algorithm~\ref{alg:scope} (Stages~0--2).
Let $d_m$ be the maximum degree of the screened graph encoded by the Stage~0 mask $\widehat M$
and let $\widehat M_\pi:=P_\pi \widehat M P_\pi^\top$ be the permuted mask.
We bound the runtime by summing the stagewise costs.

\paragraph{Stage 0 (precision estimation).}
When Stage~0 is implemented via DRO-calibrated graphical lasso, the cost has two components:
\begin{itemize}
    \item $B$ bootstrap replicates, costing $O(Bnp^2)$ in total to form the sample covariance matrices $\{\widehat{\Sigma}^{(b)}\}_{b=1}^B$.
    \item One graphical lasso solve with $T_{\mathrm{GL}}$ outer iterations, costing $O(T_{\mathrm{GL}}p^3)$ in a dense worst-case analysis,
    since the dominant subroutines reduce to dense matrix factorizations; see \citep{friedman2008glasso,hsieh2013bigquic,hsieh2014quic}.
\end{itemize}
Therefore,
\[
T_0 = O\!\left(Bnp^2 + T_{\mathrm{GL}}p^3\right).
\]

\paragraph{Stage 1 (masked IC(0)).}
We show that masked IC(0) runs in $O(p\,d_m^2)$ time under the degree bound $d_m$.
Write $\widehat{L}_\pi=\mathrm{IC0}(\widehat\Omega_\pi;\widehat M_\pi)$ for Stage~1 output.
For each row index $k\in[p]$, define the active predecessor set
\[
S_k := \{j<k:\; (\widehat M_\pi)_{kj}=1\}.
\]
By definition of $d_m$ as a maximum degree bound for the screened graph, we have $|S_k|\le d_m$ for all $k$.

Masked IC(0) computes, for each $k$, the entries $\{(\widehat{L}_\pi)_{kj}: j\in S_k\}$ and the diagonal $(\widehat{L}_\pi)_{kk}$
using only indices in $S_k$. Concretely, for each fixed pair $j\in S_k$,
the IC(0) update for $(\widehat{L}_\pi)_{kj}$ is of the generic form
\[
(\widehat{L}_\pi)_{kj}
\;=\;
\frac{1}{(\widehat{L}_\pi)_{jj}}
\Bigl(
(\widehat\Omega_\pi)_{kj}
-\sum_{\ell\in S_k\cap S_j} (\widehat{L}_\pi)_{k\ell}(\widehat{L}_\pi)_{j\ell}
\Bigr),
\]
and the diagonal update has the form
\[
(\widehat{L}_\pi)_{kk}^2
\;=\;
(\widehat\Omega_\pi)_{kk}
-\sum_{\ell\in S_k} (\widehat{L}_\pi)_{k\ell}^2,
\]
with the breakdown condition corresponding to $(\widehat{L}_\pi)_{kk}^2\le 0$.

We now count arithmetic operations. For each $k$:
\begin{itemize}
\item Computing $(\widehat{L}_\pi)_{kk}$ requires a sum over $t\in S_k$ and thus costs $O(|S_k|)\le O(d_m)$ operations.
\item For each $j\in S_k$, computing $(\widehat{L}_\pi)_{kj}$ requires summing over $S_k\cap S_j$.
Since $S_k\cap S_j \subseteq S_k$, this costs $O(|S_k\cap S_j|)\le O(|S_k|)\le O(d_m)$ operations.
There are $|S_k|\le d_m$ such indices $j$.
\end{itemize}
Hence the total cost at row $k$ is
\[
O(d_m) + \sum_{j\in S_k} O(d_m)
\;=\;
O(d_m) + O(|S_k|\,d_m)
\;=\;
O(d_m^2).
\]
Summing over $k=1,\dots,p$ yields
\[
T_1 = \sum_{k=1}^p O(d_m^2) = O(p\,d_m^2).
\]

\paragraph{Stage 2 (\textsc{RefitTest}).}
Stage~2 performs $p$ local regressions. For each node $k$, the candidate predictor set is the parent-candidate set
implied by the support of the Stage~1 factor, hence its size is at most the screened degree bound:
\[
s_k := |\widehat{\pa}_\pi(k)| \le d_m.
\]
Consider the regression of $X_k$ on $X_{\widehat{\pa}_\pi(k)}$ with $s_k$ predictors.
Forming the Gram matrix and cross-product costs
\[
X^\top X:\ O(n s_k^2),
\qquad
X^\top y:\ O(n s_k).
\]
Solving the resulting $s_k\times s_k$ linear system by a generic dense method costs $O(s_k^3)$.
Under the stated regime $d_m<n$, we have $s_k\le d_m<n$, so $s_k^3 \le s_k^2 n$ and the solve is dominated by $O(n s_k^2)$.
Computing residuals and $t$-statistics adds at most $O(n s_k + s_k^2)$ and is also dominated by $O(n s_k^2)$.
Therefore, the total cost per node is $O(n s_k^2)\le O(n d_m^2)$ for each $k$. Summing over $k=1,\dots,p$ gives
\[
T_2 = O(p\,n\,d_m^2).
\]

Combining the stagewise bounds,
\[
T(\pi)=T_0+T_1+T_2
=
O\!\left(Bnp^2 + T_{\mathrm{GL}}p^3 + p\,d_m^2 + p\,n\,d_m^2\right).
\]
Since $n\ge 1$, the term $p\,d_m^2$ is absorbed by $p\,n\,d_m^2$, yielding
\[
T(\pi)= O\!\left(B\,n\,p^2 + T_{\mathrm{GL}}\,p^3 + p\,n\,d_m^{\,2}\right),
\]
which proves the theorem.

\section{Numerical Experiments}
\label{app:experiments}

\paragraph{Implementation details.}
Within \textsc{SCOPE}, we set $\alpha=0.01$ for both Stage~0 selection of the graphical lasso regularization parameter $\lambda$ and Stage~2 refit-and-test step. For Stage~0, the precision matrix $\widehat{\Omega}$ is estimated by graphical lasso using the SQUIC solver~\citep{bollhofer2019squic}, and $\lambda$ is selected by the Robust Selection rule via the \texttt{robust-selection} Python package~\citep{cisnerosvelarde2020droglasso,robustselection2020pypi}. If SQUIC fails (e.g., due to an internal assertion when the returned solution is not positive definite), we skip the corresponding trial. Unless otherwise stated, all remaining hyperparameters and solver settings are kept at their default values, as in the respective references.

\paragraph{Evaluation.}
For methods that output a DAG, we convert the estimated DAG to a CPDAG using the \texttt{causaldag} Python package~\citep{squires2018causaldag}.
In rare cases where the estimated DAG becomes excessively dense, this CPDAG conversion can be prohibitively slow; if the conversion exceeds 60 seconds, we terminate the conversion and use the DAG skeleton as a CPDAG proxy.
Results are averaged over 30 replicates for settings with $p<10^4$ and over 10 replicates for the largest-scale setting with $p=10^4$.

\paragraph{Synthetic linear SEM generation.}
Synthetic data is generated from a linear SEM~\eqref{eq:sem_model} with independent errors.
To match the sparsity regime considered throughout the paper, we sample bounded in-degree DAGs by letting each node select at most $d$ parents among the earlier nodes in a random ordering.
For graphs with $p>1000$, we generate DAGs in a blockwise manner using blocks of size at most $1000$.
For each selected edge, the coefficient is drawn uniformly from $\pm[0.6,0.8]$.
Error variances are sampled independently from $[0.8,1.0]$.
To assess robustness beyond Gaussianity, the noise terms are drawn from one of the following distributions: standard normal, Student-$t$ with $10$ degrees of freedom, or uniform.
In all cases, the noise is centered to have mean zero and then rescaled so that its variance matches the sampled value in $[0.8,1.0]$.

We summarize below the software used for each baseline and the corresponding user-specified hyperparameters, with all unspecified options left at their package defaults.

\begin{itemize}
    \item \textbf{PC}~\citep{kalisch2007pc}.
    
    Implementation: \texttt{pcalg::pc} (R) via \texttt{rpy2}~\citep{kalisch2012pcalg}.
    
    The algorithm uses the Gaussian Fisher-$z$ CI test with $\alpha=0.01$.
    We set the maximum conditioning-set size to \texttt{m.max}$=3$ and use the original skeleton routine without the stable variant to keep the computational cost practical at our problem sizes.

    \item \textbf{GES}~\citep{chickering2002ges}.
    
    Implementation: \texttt{pcalg::ges} (R) with \texttt{GaussL0penObsScore} via \texttt{rpy2}~\citep{kalisch2012pcalg}.
    
    We use the standard $\ell_0$-penalized Gaussian score with penalty $\lambda = 0.5\log(n)$ (i.e., \texttt{lambda\_mult}$=0.5$), and otherwise keep defaults.

    \item \textbf{GSP}~\citep{solus2021consistency}.
    
    Implementation: \texttt{causaldag} (Python)~\citep{squires2018causaldag}.
    
    We run GSP with a partial-correlation CI tester using $\alpha=0.01$.
    In the \texttt{causaldag} interface, we do not provide \texttt{fixed\_adjacencies} or \texttt{fixed\_gaps} (so both remain \texttt{None}), and we keep the remaining options at their defaults.

    \item \textbf{NOTEARS}~\citep{zheng2018notears}.
    
    Implementation: \texttt{gCastle} \texttt{Notears} (Python)~\citep{zhang2021gcastle}.
    
    We use \texttt{lambda1}$=0.1$, \texttt{loss\_type}=\texttt{"l2"}, \texttt{max\_iter}$=20$ (gCastle default), and otherwise use defaults. If the estimated graph is not acyclic, we remove cycles by edge pruning, following the post-processing recommended in the NOTEARS paper (e.g., iteratively deleting the weakest edges until a DAG is obtained).

    \item \textbf{GOLEM}~\citep{ng2020golem}.
    
    Implementation: \texttt{gCastle} \texttt{GOLEM} (Python)~\citep{zhang2021gcastle}.
    
    We use \texttt{lambda\_1}$=0.1$, \texttt{lambda\_2}$=0.1$, \texttt{equal\_variances}=\texttt{False}, \texttt{num\_iter}$=1000$ (gCastle default), and otherwise use defaults. If the estimated graph is not acyclic, we apply the same edge-pruning post-processing as for NOTEARS, following the procedure recommended in the GOLEM paper.

    \item \textbf{FROSTY}~\citep{bang2023frosty} and \textbf{RFD}~\citep{squires2020rfd}.
    
    Implementation: Python implementations provided by the authors.
    
    Both methods take a precision estimate as input.
    We obtain $\widehat{\Omega}$ by graphical lasso with $\lambda$ selected by the Robust Selection rule at $\alpha=0.01$, following the tuning protocol in \citet{bang2023frosty}.
    We solve the graphical lasso using the SQUIC implementation~\citep{bollhofer2019squic}, and compute $\lambda$ using the \texttt{robust-selection} Python package~\citep{cisnerosvelarde2020droglasso,robustselection2020pypi}.
\end{itemize}

\subsection{Comparison with Baselines in low dimensional settings}
\label{app:exp_comparison}

\begin{figure*}[t]
    \centering
    \includegraphics[width=0.98\linewidth]{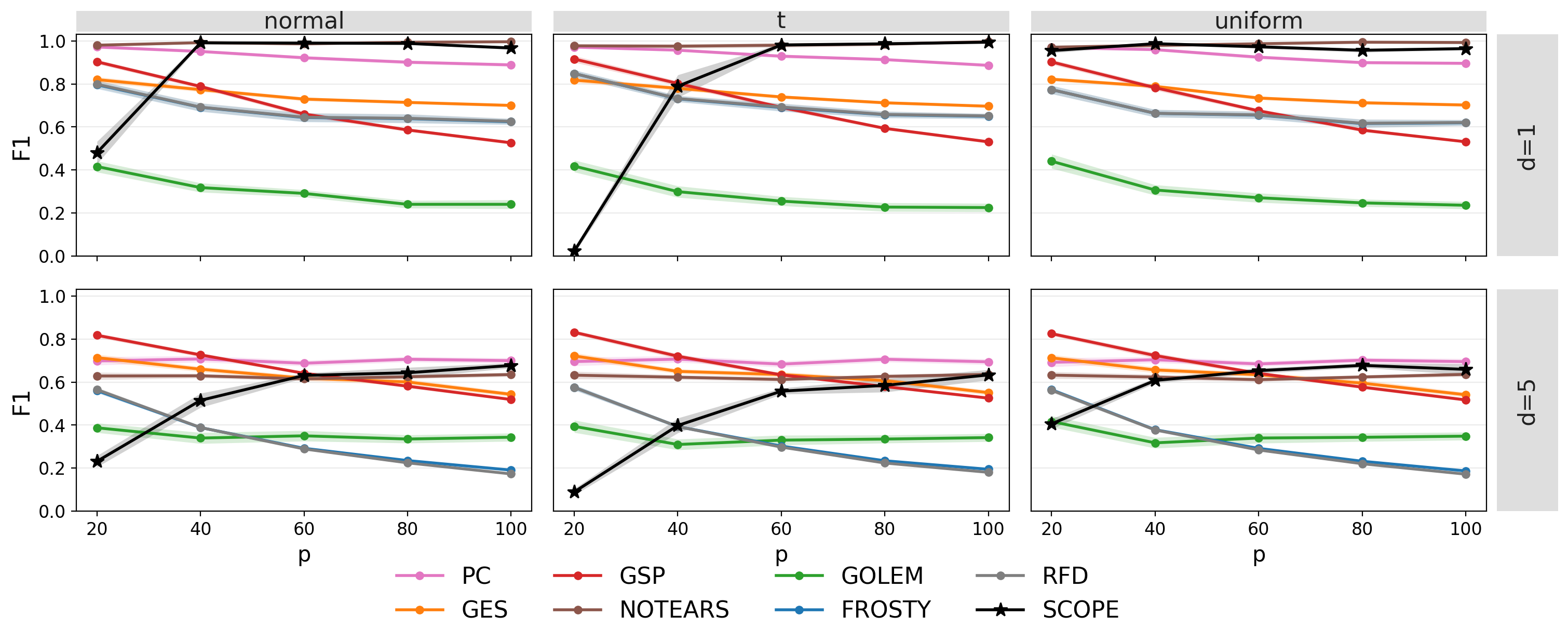}
    \caption{CPDAG-skeleton $F_1$. Shaded bands represent $\pm$1 standard error around the mean across 30 replicates.}
    \label{fig:metrics_vs_p:f1}
\end{figure*}

\begin{figure*}[!t]
    \centering
    \includegraphics[width=0.98\linewidth]{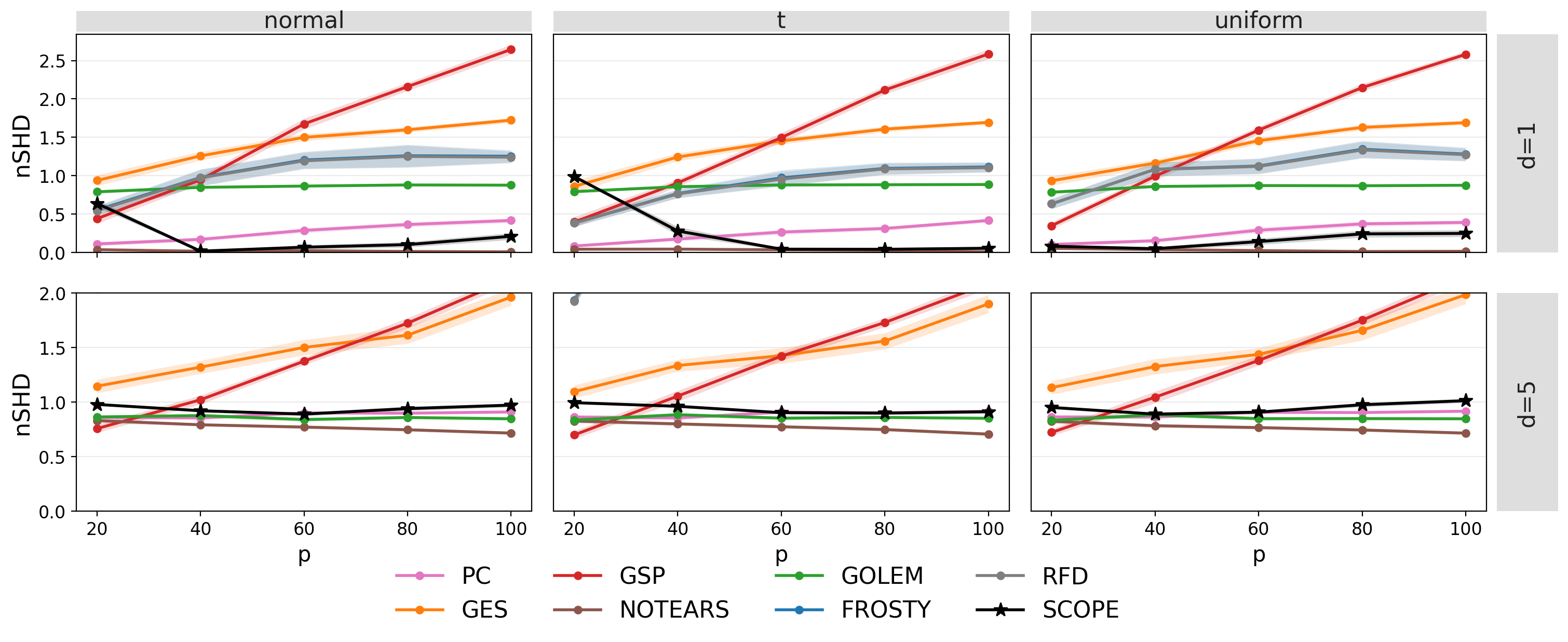}
    \caption{nSHD for CPDAG recovery. Shaded bands represent $\pm$1 standard error around the mean across 30 replicates.}
    \label{fig:metrics_vs_p:nshd}
\end{figure*}

We compare \textsc{SCOPE} under the AMD heuristic ordering with structure learning baselines on synthetic linear SEMs, varying
$p\in\{20,40,60,80,100\}$ with $n=20p$ and bounded indegree $d\in\{1,5\}$.
Figures~\ref{fig:metrics_vs_p:f1}--\ref{fig:metrics_vs_p:nshd} report CPDAG-skeleton $F_1$ and nSHD for CPDAG recovery under three noise distributions (normal, $t$, and uniform).
Figure~\ref{fig:accuracy_runtime}, presented earlier in the main text, shows only the Gaussian-case results for $p=100$ and $d=5$, which form a subset of the results considered in Figure~\ref{fig:metrics_vs_p:f1}.
Across all distributions, \textsc{SCOPE} maintains strong accuracy as $p$ increases, indicating that its performance
is not tied to a particular noise model.

In the zero-fill-friendly regime ($d=1$), each node has at most one parent, so moralization introduces no co-parent links and the moralized graph coincides with the DAG skeleton. Moreover, this graph is chordal and admits a perfect elimination ordering~\citep{rose1976elimination,george1981computer}.
Accordingly, sparse elimination can proceed with essentially zero fill, and \textsc{SCOPE} achieves near-exact
CPDAG recovery when precision estimation noise is modest.
When $d=5$, moralization-induced densification makes the ordering problem harder. However, the precision-support mask restricts admissible fill-in and suppresses elimination-induced artifacts, resulting in only moderate degradation in both $F_1$ and nSHD.

Relative to baselines, \textsc{SCOPE} matches or exceeds PC in skeleton recovery while avoiding the rapidly growing
burden of CI testing as $p$ increases.
Among continuous-optimization methods, NOTEARS can be accurate but is substantially slower, while GOLEM is faster
yet often loses accuracy.
Compared with greedy search methods (GES/GSP), \textsc{SCOPE} consistently attains smaller nSHD, indicating more reliable
CPDAG recovery in these regimes, and remains stable as $p$ and $d$ increase.
Finally, the sparse-Cholesky baselines (FROSTY and RFD) tend to suffer from false discoveries—especially in the
$d=5$ regime—leading to substantially larger nSHD.

We also observe small but systematic distribution-dependent differences.
Relative to normal and $t$ noise, uniform noise can yield slightly stronger performance at smaller $p$ for several methods, consistent with lighter-tailed sampling fluctuations under the fixed scaling $n=20p$.
In addition, in the nSHD plots under normal and uniform noise, nSHD exhibits a mild upward drift as $p$ increases.
A plausible explanation is that, as $n$ grows with $p$, the plug-in precision estimate $\widehat{\Omega}$ can become less sparse and induce a slightly denser screening mask $\widehat{M}$. This in turn expands the admissible region for fill and can introduce a small increase in CPDAG discrepancy.
Since these effects are largely mediated through Stage~0, their exact magnitude may vary with the choice of the plug-in precision estimator.

\subsection{High-Dimensional Performance}
\label{app:exp_scalability}

\begin{table*}[t]
\centering
\caption{Scalability results for high-dimensional synthetic SEMs under normal and uniform noise. Performance is evaluated on the estimated CPDAG.}
\label{tab:scalability_normal_uniform}

\begin{subtable}[t]{0.48\textwidth}
\centering
\subcaption{Normal noise ($d=1$)}
\begin{tabular}{c c c c}
\hline
$n$ & Method & $F_1$ & SHD \\
\hline
$0.5p$ & SCOPE  & \textbf{0.965} (0.006) & \textbf{0.163} (0.078) \\
$0.5p$ & FROSTY & 0.627 (0.010) & 1.193 (0.050) \\
$p$    & SCOPE  & \textbf{0.911} (0.005) & \textbf{0.345} (0.109) \\
$p$    & FROSTY & 0.626 (0.008) & 1.196 (0.042) \\
\hline
\end{tabular}
\end{subtable}
\hfill
\begin{subtable}[t]{0.48\textwidth}
\centering
\subcaption{Normal noise ($d=5$)}
\begin{tabular}{c c c c}
\hline
$n$ & Method & $F_1$ & SHD \\
\hline
$0.5p$ & SCOPE  & \textbf{0.580} (0.204) & \textbf{1.166} (0.074) \\
$0.5p$ & FROSTY & 0.022 (0.002) & 88.658 (12.876) \\
$p$    & SCOPE  & \textbf{0.562} (0.021) & \textbf{1.864} (0.109) \\
$p$    & FROSTY & 0.021 (0.000) & 93.024 (1.093) \\
\hline
\end{tabular}
\end{subtable}

\vspace{0.8em}

\begin{subtable}[t]{0.48\textwidth}
\centering
\subcaption{Uniform noise ($d=1$)}
\begin{tabular}{c c c c}
\hline
$n$ & Method & $F_1$ & SHD \\
\hline
$0.5p$ & SCOPE  & \textbf{0.948} (0.003) & \textbf{0.268} (0.062) \\
$0.5p$ & FROSTY & 0.622 (0.011) & 1.218 (0.056) \\
$p$    & SCOPE  & \textbf{0.903} (0.007) & \textbf{0.400} (0.075) \\
$p$    & FROSTY & 0.624 (0.011) & 1.207 (0.056) \\
\hline
\end{tabular}
\end{subtable}
\hfill
\begin{subtable}[t]{0.48\textwidth}
\centering
\subcaption{Uniform noise ($d=5$)}
\begin{tabular}{c c c c}
\hline
$n$ & Method & $F_1$ & SHD \\
\hline
$0.5p$ & SCOPE  & \textbf{0.642} (0.020) & \textbf{1.241} (0.059) \\
$0.5p$ & FROSTY & 0.022 (0.000) & 85.750 (0.638) \\
$p$    & SCOPE  & \textbf{0.497} (0.175) & \textbf{1.835} (0.302) \\
$p$    & FROSTY & 0.020 (0.002) & 97.246 (10.738) \\
\hline
\end{tabular}
\end{subtable}

\end{table*}

Table~\ref{tab:scalability_normal_uniform} reports additional $p=10^4$ scalability results under normal and uniform noise.
Results under normal and uniform noise align with those under $t$ noise in Table~\ref{tab:scalability}.
In the zero-fill-friendly regime ($d=1$), \textsc{SCOPE} under the AMD heuristic ordering achieves high $F_1$ with small nSHD across both $n=0.5p$ and $n=p$.
When $d=5$ and moralization densifies the screened structure, \textsc{SCOPE} remains markedly more stable than FROSTY, whose $F_1$ collapses and nSHD becomes extremely large.
We also observe that increasing $n$ does not necessarily improve accuracy, reflecting the plug-in nature of Stage~0 through the induced screening mask.
Overall, \textsc{SCOPE} is robust across noise distributions and continues to yield usable CPDAG estimates at large scale.

\section{Breast Cancer Data Analysis}
\label{app:real}

\paragraph{Data description.}
We illustrate the performance of our method for analyzing a real mRNA expression dataset from \cite{Koh2019}. See data files on authors' public GitHub repository: \url{https://github.com/Hiromikwl/DataCodes_iOmicsPASS}.

The transcriptomic data consists of mRNA expression profiles derived from invasive ductal breast carcinoma samples from The Cancer Genome Atlas (TCGA) project \citep{TheCancerGenomeAtlasNetwork2012}. The samples are stratified into four intrinsic phenotypic subgroups-Luminal A, Luminal B, HER2E, and Basal-like-determined by the mRNA-based PAM50 signature \citep{Parker2009}.

The data were originally acquired via the Genomics Data Commons (GDC) data portal and processed using GDC pipelines aligned to the GRCh38 reference genome, and the original ENSEMBL identifiers were mapped and converted to HGNC gene symbols \citep{Koh2019}. The mRNA data were mapped to a transcription factor (TF) regulatory network consisting of 2,486 TFs and 14,796 target genes. For broader pathway enrichment analysis, the study utilized a background set of 17,250 genes derived from ConsensusPathDB \citep{Kamburov2011} and Gene Ontology biological processes \citep{Ashburner2000}.

\paragraph{Implementation details.}
As described in Appendix~\ref{app:sample}, we calibrate the graphical lasso regularization parameter $\lambda$ via Robust Selection (RobSel).
For the extremely high-dimensional TCGA mRNA dataset with $(n,p)=(599,16325)$, we follow prior work \citep{robustselection2020pypi,bang2023frosty} and set level $\alpha^{(0)}=0.99$ for Stage 0, as RobSel tends to select a conservative regularization parameter $\lambda$ in this regime.

In addition, since this high-dimensional setting already yields a relatively strong regularization level through the DRO calibration, we use $\widehat{\Omega}(\lambda)$ directly as the Stage~0 precision estimate, instead of the thresholded version $\widehat{\Omega}^{\mathrm{thr}}$ in \eqref{eq:Omega_thr}.
All other implementation details for \textsc{SCOPE} follow Appendix~\ref{app:experiments}.

The implementations of PC and GES are also identical to those described in Appendix~\ref{app:experiments}.

\begin{table}[t]
\centering
\begin{tabular}{l l c c c c c}
\toprule
& & & \multicolumn{2}{c}{One-step Paths} & \multicolumn{2}{c}{Two-step Paths} \\
\cmidrule(lr){4-5} \cmidrule(lr){6-7}
Method & Random partition & \# of directed edges & Fwd/Rev & Ratio & Fwd/Rev & Ratio \\
\midrule
PC & 1 & 62,422 & 0.028/0.024 & 1.167 & 0.051/0.045 & 1.133 \\
PC  & 2 & 62,515  & 0.026/0.024 & 1.083 & 0.050/0.047 & 1.064 \\
PC  & 3 & 62,708  & 0.027/0.025 & 1.080 & 0.051/0.046 & 1.109 \\
GES & 1 & 308,533 & 0.028/0.024 & 1.167 & 0.056/0.045 & 1.244 \\
GES & 2 & 308,707 & 0.028/0.022 & 1.273 & 0.055/0.045 & 1.222 \\
GES & 3 & 307,769 & 0.029/0.022 & 1.318 & 0.056/0.045 & 1.244 \\
\midrule
\textsc{SCOPE} &  & 271,066 & 0.078/0.045 & \textbf{1.733} & 0.108/0.061 & \textbf{1.770} \\
\bottomrule
\end{tabular}
\caption{Proportions of biologically validated TF$\to$target paths recovered in the estimated CPDAGs for additional random partitions. \textsc{SCOPE} is evaluated on the full dataset and is therefore unchanged across partitions.}
\label{tab:tcga_tfdb_hops_compact_random_partitionings}
\end{table}

\paragraph{Results on additional random partitions.}
Table~\ref{tab:tcga_tfdb_hops_compact_random_partitionings} reports the PC and GES results for additional random partitions of the features; random partition 1 matches the result in Table~\ref{tab:tcga_tfdb_hops_compact}.

Across the additional random partitions, the same qualitative conclusion continues to hold. Although the split-based baselines show small partition-dependent fluctuations in both validations (one-step ratio ranges: PC, $1.080$--$1.167$; GES, $1.167$--$1.318$; two-step ratio ranges: PC, $1.064$--$1.133$; GES, $1.222$--$1.244$), \textsc{SCOPE} remains clearly separated from both methods, with substantially larger forward-to-reverse ratios in both one-step and two-step validations ($1.733$ and $1.770$, respectively). This qualitative gap is also consistent with the overall fraction of externally validated associations (regardless of direction): \textsc{SCOPE} attains $0.292$, whereas PC and GES remain in the ranges $0.147$--$0.149$ and $0.150$--$0.153$, respectively.

\end{document}